%% file: neurips_2026.tex
\documentclass{article}

 \usepackage[main, final]{neurips_2026}

\usepackage{xcolor}    
\usepackage{listings}  
 
\definecolor{codegreen}{rgb}{0,0.6,0}
\definecolor{codegray}{rgb}{0.5,0.5,0.5}
\definecolor{codepurple}{rgb}{0.58,0,0.82}
\definecolor{backcolour}{rgb}{0.95,0.95,0.92}

\lstdefinestyle{mystyle}{
    commentstyle= \color{red!50!green!50!blue!50},  
    keywordstyle= \color{blue!70},  
    numberstyle=\tiny\color{codegray},  
    stringstyle=\color{codepurple},
    basicstyle=\ttfamily\footnotesize,
    breakatwhitespace=false,
    breaklines=true,  
    captionpos=b,
    keepspaces=true,
    numbers=left,  
    numbersep=5pt,
    showspaces=false,
    showstringspaces=false,  
    showtabs=false,
    tabsize=2,
    frame=single  
}
 
\usepackage[utf8]{inputenc} 
\usepackage[T1]{fontenc}    
\usepackage{hyperref}       
\usepackage{url}            
\usepackage{booktabs}       
\usepackage{amsfonts}       
\usepackage{nicefrac}       
\usepackage{microtype}      
\usepackage{xcolor}         

\usepackage{amsmath}
\usepackage{amssymb}
\usepackage{mathtools}
\usepackage{amsthm}
\usepackage{bbm}

\usepackage[T1]{fontenc}
\usepackage[utf8]{inputenc}

\theoremstyle{plain}
\newtheorem{theorem}{Theorem}
\newtheorem{lemma}{Lemma}

\theoremstyle{definition}

\theoremstyle{remark}

\usepackage[utf8]{inputenc} 
\usepackage[T1]{fontenc}    
\usepackage{hyperref}       
\usepackage{url}            
\usepackage{booktabs}       
\usepackage{amsfonts}       
\usepackage{nicefrac}       
\usepackage{microtype}      
\usepackage{xcolor}         
\input{math_command}

\usepackage{caption}
\usepackage{wrapfig}
\usepackage{makecell}
\usepackage{url}
\usepackage{float}
\usepackage{kotex}
\usepackage{bm}
\usepackage{multirow}
\usepackage{xcolor}
\usepackage{colortbl}
\usepackage{tabulary}
\usepackage{listings}
\usepackage{multicol}
\usepackage{xspace}
\usepackage{graphbox}
\usepackage{arydshln}
\usepackage{adjustbox}
\usepackage{diagbox}
\usepackage{subcaption}
\usepackage{setspace}  
\usepackage{algorithm}
\usepackage{algpseudocode}

\usepackage{tcolorbox}
\tcbuselibrary{skins, breakable}

\title{Opening the Black Box of Classifier-Free Guidance \\via Information Bottleneck}

\author{
Jiayang Gao$^{1,2,\ast}$, Tianyi Zheng$^{2,\ast}$, Jiayang Zou$^3$, \\
\textbf{Fengxiang Yang$^2$, Luyao Fan$^1$, Lv Tang $^4$, Bo Li$^{2,\dagger}$, Jia Wang$^{1,\dagger}$ }
  \\$^1$ {Shanghai Jiao Tong University, Shanghai, China} 
  \\
  $^2${vivo BlueImage Lab, vivo Mobile Communication Co., Ltd., China} \\
  $^3$ {Stanford University, California, USA} 
  \\
  $^4$ {Michael Smith Laboratories, University of British Columbia} 
  \\
  {\tt\small {\{gjy0515, jiawang\}}@sjtu.edu.cn, \{zhengtianyi, libra\}@vivo.com, jyangzou@stanford.edu}}

\begin{document}
\addtocontents{toc}{\protect\setcounter{tocdepth}{0}}

\maketitle
\footnotetext{$\ast$ These authors contributed equally to this work.}
\footnotetext{$\dagger$ Corresponding authors.}

\input{sec/abstract}
\input{sec/intro}

\input{sec/background}
\input{sec/theory}

\input{sec/method}

\input{sec/result}

\input{sec/conclusion}

\bibliography{neurips}
\bibliographystyle{unsrt}

\newpage
\appendix
\newpage
\addtocontents{toc}{\protect\setcounter{tocdepth}{2}}   
\hypersetup{linkcolor=black}
\newpage
\section*{\centering\LARGE Appendix}

\begingroup
\setstretch{1.4}   
\large
\tableofcontents
\endgroup

\hypersetup{linkcolor=red}
\newpage

\appendix
\input{sec/appendix}

\clearpage
\newpage

\input{checklist.tex}

\end{document}

%% file: math_command.tex
\usepackage{amsmath,amsfonts,bm}

\def\eqref#1{equation~\ref{#1}}

\def\1{\bm{1}}

\def\rd{{\textnormal{d}}}

\def\rmI{{\mathbf{I}}}

\DeclareMathAlphabet{\mathsfit}{\encodingdefault}{\sfdefault}{m}{sl}
\SetMathAlphabet{\mathsfit}{bold}{\encodingdefault}{\sfdefault}{bx}{n}

\def\rmI{{\mathbf{I}}}

\newcommand{\innerprod}[2]{\left<{#1},{#2}\right>}

\newcommand{\x}{{\boldsymbol x}}

\newcommand{\g}{{\boldsymbol g}}

\newcommand{\cb}{{\boldsymbol c}}

\newcommand{\epsilonb}{{\boldsymbol \epsilon}}

\definecolor{cfg}{rgb}{0.906, 0.435, 0.318}
\definecolor{cfgpp}{rgb}{0.165, 0.616, 0.561}
\definecolor{cfgnull}{rgb}{0.208, 0.565, 0.953}

\newcommand{\epscond}{{\hat\epsilonb_\cb}}

\newcommand{\epsnullnc}{{\hat\epsilonb_\varnothing}}
\newcommand{\epscfgnc}{{\hat\epsilonb_\cb^\omega}}


%% file: sec/abstract.tex
\begin{abstract}
Classifier-free guidance (CFG) is a standard technique for improving conditional generation in diffusion models, yet its activation schedule is usually selected heuristically. In practice, most existing CFG-based methods rely on a pre-defined activation policy throughout inference, while recent dynamic variants often depend on hand-designed schedules without a principled theoretical explanation of when conditional information is most useful. 
In order to understand how the informational relevance between conditioning inputs and the score function evolves over the course of generation, we explore when guidance should be activated from an information-theoretic perspective. In this work, we formulate the diffusion trajectory as a time-dependent information bottleneck and show that the effectiveness of conditional information is not uniform throughout sampling, but becomes significantly stronger after a critical stage characterized by a Fisher-information balance condition.
Motivated by this insight, we propose a training-free dynamic CFG strategy for inference. Rather than prescribing a fixed schedule in advance, our method constructs a stochastic process from conditional and unconditional score statistics and activates guidance when the estimated guidance gain crosses a prescribed threshold. This yields an adaptive activation rule that aligns condition injection with the underlying generation dynamics, while also recovering interval-based guidance as a limiting special case.
Extensive experiments across diverse architectures and tasks, demonstrate that our method is effective, robust, and broadly applicable. Beyond empirical improvements, our framework provides an interpretation of dynamic guidance methods and a principled foundation for inference-time control in conditional diffusion models.
\end{abstract}

%% file: sec/intro.tex
\section{Introduction}
Diffusion models \citep{SohlDickstein2015DeepUL,Song2019GenerativeMB,Song2020DenoisingDI,Song2020ScoreBasedGM} have become a major paradigm for high-quality image generation \cite{Dhariwal2021DiffusionMB}, text-to-image synthesis \cite{rombach2022high} and 3D generation \citep{woo2024harmonyview}. With the development of conditional diffusion models, the generation process can be effectively controlled by class labels \cite{Ho2022ClassifierFreeDG, Dhariwal2021DiffusionMB}, textual descriptions \cite{balaji2022ediff,ramesh2022hierarchical}, or other forms of prior information. Among existing conditioning methods, Classifier-Free Guidance (CFG) \cite{Ho2022ClassifierFreeDG} has become one of the most widely used mechanisms for conditional generation. Most existing CFG-based methods are applied with a fixed guidance strategy throughout inference, implicitly assuming that conditional information is equally important across all denoising stages. However, this assumption may be overly simplistic, as the relevance of conditional information is likely to evolve over time, which is a phenomenon that remains insufficiently understood. 
Although several recent methods have attempted to improve this design through dynamic guidance strategies, such as time-dependent reweighting \cite{malarz2025classifier, poleski2025geoguide, Gao_2026_CVPR,11143960}, fixed interval activation \cite{kynkaanniemi2024applying}, or modified guidance dynamics \cite{wang2026cfg, Sadat2025GuidanceIT}, these approaches still rely largely on heuristic schedules.
In particular, it remains unclear how the informational relevance between conditioning inputs and the score function evolves during generation, and hence when conditional guidance should be activated most effectively.

Diffusion models generate samples by learning a reverse denoising process that progressively reconstructs the target data distribution. For conditional generation, the key question is how to inject conditioning information into this process effectively. Early approaches such as classifier guidance \cite{Dhariwal2021DiffusionMB}, as well as more recent TFG-style methods \cite{ye2024tfg,Shen2024UnderstandingAI}, use external classifiers to provide conditional gradients, but they often struggle to remain stable under heavy noise. By contrast, CFG \cite{Ho2022ClassifierFreeDG}  incorporates conditional information directly into training by jointly learning conditional and unconditional branches, making it more robust and practically appealing. Based on CFG, a series of variant methods replaces the unconditional output with alternative degraded references, such as the prediction of an under-trained conditional network \cite{Karras2024autoguidance} or that of a conditional network with perturbed conditioning inputs \cite{rajabi2025token}. 
These approaches preserve the general guidance form of comparing a target conditional prediction against a weaker reference, while modifying how the reference branch is constructed.

Taken together, these conditioning and guidance strategies highlight that the effectiveness of guidance depends not only on how conditional information is injected, but also on when it becomes most relevant to the denoising dynamics. 
This raises a more fundamental question: \textbf{what are the benefits of properly scheduling the injection of conditional information, and how can this be understood from a theoretical perspective?} 
To answer this, we investigate how conditional information propagates through a conditional diffusion model, and identify when it becomes most influential. We approach this problem through the lens of the information bottleneck, enabling a principled analysis of how conditioning signals are preserved and utilized across the denoising trajectory.
In particular, we study the relationship between conditioning inputs and the score function, and establish a connection between mutual information and score-related statistics. Our theoretical analysis reveals that the effectiveness of conditional information is not uniform throughout inference; instead, it becomes significantly stronger only after a critical stage. Correspondingly, the discrepancy between conditional and unconditional predictions, together with changes in score norms, provides observable signatures of this phenomenon. These theoretical findings are consistent with empirical observations, suggesting that uniformly applying a fixed CFG strategy across the entire inference process is suboptimal.

Building on these insights, we propose a dynamic CFG switching strategy for inference. Our key idea is to activate conditional guidance only when the generation process reaches the information bottleneck, so that the timing of condition injection is explicitly aligned with the theoretical analysis. 
The method is training-free and can be readily applied to a broad range of diffusion models and generation tasks, establishing a principled design criterion for CFG-based generation highlighting the differences and connections between various dynamic guidance strategies. In particular, it clarifies Interval Guidance \cite{kynkaanniemi2024applying} can be viewed through the same theoretical lens, offering insights into how we choose time to activate CFG.
Extensive experiments on SiT \cite{ma2024sit, yu2025repa}, DiT \cite{Peebles2022DiT}, FLUX \cite{labs2025flux1kontextflowmatching}, and Stable Diffusion \cite{rombach2022high, esser2024scalingrectifiedflowtransformers} show that our approach is both effective and broadly applicable.
Our contributions are summarized as follows:

\begin{itemize}
    \item \textbf{Information-Theoretic Perspective:}
    We introduce a time-dependent information-bottleneck
    framework to investigate the role of conditional
    information throughout the diffusion trajectory.
    By connecting mutual information evolution with
    Fisher information, we derive a balance condition
    that characterizes an information-theoretically
    motivated transition region for CFG activation.

    \item \textbf{Principled Activation Criterion:}
    Building on the theoretical analysis, we derive
    an observable guidance-gain statistic from the
    difference between conditional and unconditional
    score norms. We further develop a rescaled
    threshold criterion that improves numerical
    conditioning and provides an interpretable,
    model-dependent mechanism for identifying
    effective guidance activation regions.

    \item \textbf{Adaptive Guidance Strategy:}
    We propose a training-free and plug-and-play
    dynamic CFG strategy that determines the
    activation time online through a threshold-crossing
    rule. Our framework accommodates both batch-level
    and single-sample inference, while recovering
    fixed-interval guidance in the large-batch
    concentration regime, thereby establishing
    a connection between adaptive and
    predetermined guidance schedules.

    \item \textbf{Extensive Empirical Validation:}
    We conduct comprehensive experiments across
    diverse diffusion and flow-based architectures,
    including SiT, DiT, SD3, and FLUX. The results
    demonstrate improvements in generation quality
    under various evaluation settings, supported
    by multi-seed stability analyses, alternative
    activation criteria, and high-resolution
    single-sample inference experiments.
\end{itemize}

%% file: sec/background.tex
\section{Background}\label{background}
\paragraph{Diffusion Models and CFG.}
Diffusion models \citep{SohlDickstein2015DeepUL,Song2019GenerativeMB,Song2020DenoisingDI,Song2020ScoreBasedGM} generate data by reversing a gradual noising process. In the continuous-time formulation, the forward diffusion process can be described by the following stochastic differential equation (SDE) \cite{Song2020ScoreBasedGM}:
\begin{align}\label{gen_sde}
    \mathrm{d} x_t = f(x_t, t) \, \mathrm{d}t + g(t) \, \mathrm{d}w_t,
\end{align}
where $f(x_t, t)$ is the drift coefficient, $g(t)$ is the diffusion coefficient, and $w_t$ is a standard Wiener process. The corresponding reverse-time SDE \citep{anderson1982reverse} is given by
\begin{align}
    \mathrm{d} x_t = \left[ f(x_t, t) - \frac{1}{2} \left( g^2(t) + \sigma^2(t) \right) \nabla \log p_t(x_t) \right] \mathrm{d}t + \sigma(t) \, \mathrm{d}\bar{w}_t,
\end{align}
where $\bar{w}_t$ denotes a standard reverse-time Wiener process, and $\sigma(t)$ specifies the noise scale. The score function $\nabla \log p_t(x_t)$ is the target quantity that the network aims to predict \cite{Song2020ScoreBasedGM}. For conditional generation, Classifier-Free Guidance (CFG) combines conditional and unconditional predictions to steer sampling:
\begin{align}\label{CFG}
\epsilon_\omega(x_t, t, c) = \omega \, \epsilon_\theta(x_t, t, c) - (\omega - 1)\epsilon_\theta(x_t, t, \varnothing),
\end{align}
where $\epsilon_\theta(x_t, t, \cdot)$ denotes the conditional noise predictor and $\omega \ge 1$ controls the guidance strength. We defer a more detailed discussion of representative guidance variants to Appendix~\ref{Related Work}.

\paragraph{Mutual Information, Fisher Information, and Information Bottleneck (IB).}
To analyze the role of conditional information, we adopt an information-theoretic perspective. The mutual information \cite{Cover2005ElementsOI} between random variables $X$ and $Y$ is defined as
\begin{align}
I(X;Y)
=
\mathbb{E}_{p(x,y)}
\left[
\log \frac{p(x,y)}{p(x)p(y)}
\right],
\end{align}
which measures their statistical dependence. We also consider Fisher information,
\begin{align}
J(X)
=
\mathbb{E}_{p(x)}
\left[
\left\|
\nabla_x \log p(x)
\right\|^2
\right],\quad 
J(X|Y)=
\mathbb{E}_{p(x,y)}
\left[
\left\|
\nabla_x \log p(x|y)
\right\|^2
\right]
\end{align}
which is directly related to score functions. Analogous to the decomposition of mutual information, the Fisher mutual information between $X$ and $Y$ can be defined as
\begin{align}
J(X; Y)
=
J(X |Y)-J(X),
\end{align}
which quantifies the amount of Fisher information about $X$ explained by conditioning on $Y$. At last we introduce information bottleneck principle \cite{tishby2000information}, which formulates representation learning as a trade-off between compression and relevance. Given random variables $X$ and $C$, IB seeks an intermediate representation $T$ that minimizes
\begin{align}\label{eq IB}
    \mathcal{L}_{\text{IB}}=I(X;T)-\beta I(T;C), \quad \beta>0,
\end{align}
where $\beta$ is a Lagrange multiplier controlling the balance between compression of the input $X$ and preservation of information relevant to the target variable $C$. Beyond its original formulation, the IB principle has also been used to interpret the training dynamics of neural networks \citep{tishby2015deep}. In this view, a neural network induces a Markov chain $Y \to X \to \hat{X}$, where $\hat{X}$ denotes the network representation or output. The learning objective can then be cast as an information bottleneck problem of the form
\begin{align}\label{eq:IBNN}
\mathcal{L}\left[p(\hat{x}| x)\right]
=
I(X; \hat{X})
-
\beta I(\hat{X}; Y),
\end{align}
which trades off the complexity of the representation $\hat{X}$ against its predictive power with respect to the labels $Y$. In our setting, this perspective provides a useful lens for understanding when conditional information becomes most important during generation.

%% file: sec/theory.tex
\section{Theoretical Analysis}\label{theory}
\subsection{Information Bottleneck on Conditional Diffusion Model}
To analyze classifier-free guidance (CFG) from an information-theoretic perspective, we view the diffusion trajectory as a Markov chain $C\to X_0\to X_t$, where $C$ represents the conditioning information (such as class labels in image classification tasks or textual prompts in text-to-image generation),  and $X_t$ is the noisy latent variable at diffusion time step $t$. Thus we can formulate CFG as a time-dependent information bottleneck problem. Specifically, we consider the objective
\begin{align} \label{eq CFG IB}
    \inf_{t} L(t) = I(X_{0};X_t)-\beta I(X_{t};C), \quad \beta >0,
\end{align}
where $t$ denotes the diffusion time and $\beta$ controls the relative emphasis on conditional information.

We first characterize the necessary condition satisfied by an interior minimizer of \eqref{eq CFG IB}.

\begin{theorem}[Information Bottleneck Optimality for Conditional Diffusion Models] \label{thm:cfg_ib}
    Assume that $X_t$ follows the forward SDE \eqref{gen_sde}, and that the mutual information in \eqref{eq CFG IB} are differentiable with respect to $t$. Assume that by choosing an appropriate $\beta$, the minimizer $t^*$ of \eqref{eq CFG IB} lies in the interior of the interval $(0, +\infty)$.Then $t^*$ satisfies
    \begin{align}
\beta J(X_{t^*}; C)=J(X_{t^*}; X_0),
\end{align}
    It can also be written as 
    \begin{align}
        \beta\left(J(X_{t^*}|C)-J(X_{t^*})\right)=J(X_{t^*}|X_0)-J(X_{t^*}).
    \end{align}
\end{theorem}
\begin{proof}
    See Appendix \ref{proof1}.
\end{proof}
We now specialize the result to an 
$n$-dimensional Ornstein–Uhlenbeck (OU) diffusion process
\begin{align}
    \rd X_t=f(t) X_t \rd t+g(t) \rd w_t,
\end{align}
whose solution is 
\begin{align}\label{eq:OU solution}
    X_t= e^{\int_0^tf(s)\rd s} X_0 + \int_0^t e^{\int_u^tf(s)\rd s} \g(u) \rd w_u. 
\end{align}
After reparameterize, $X_t$ can be written as
\begin{align}\label{eq:reparameterize}
    X_t=e^{\int_0^tf(s)\rd s} X_0 + \sqrt{m(t)} Z, 
\end{align}
where $Z\sim \mathcal{N}(0,I)$, $m(t)
=
\int_0^t
\exp\left(2\int_s^t f(r)\mathrm dr\right)
g^2(s)\mathrm ds$.
Furthermore, assuming a learned score function $s_\theta(x,t,c)\approx \nabla\log p_t(x|c)$, we approximate the Fisher information by
\begin{align}
J(X_t|C) \approx \mathbb E_{x_t,c\sim p_t(x,c)}\left[\|s_\theta(x_t,t,c)\|^2\right],\quad
J(X_t)\approx \mathbb E_{x_t,c\sim p_t(x,c)}\left[\|s_\theta(x_t,t,\varnothing)\|^2\right].
\end{align}
For such processes, the conditional Fisher information admits a closed-form expression: $J(X_t|X_0) = \frac{n}{m(t)}.$ 
Substituting these expressions into Theorem \ref{thm:cfg_ib}, the optimality condition at $t^*$ becomes
\begin{align}\label{eq:IBDiff}
    \beta (\mathbb E[\|s_\theta(x_{t^*},t^*,c)\|^2]-\mathbb E[\|s_\theta(x_{t^*},t^*,\varnothing)\|^2])=\frac{n}{m(t^*)}-\mathbb E[\|s_\theta(x_{t^*},t^*,\varnothing)\|^2].
\end{align}
Equation \ref{eq:IBDiff} characterizes the diffusion time at which the information bottleneck objective achieves a balance between conditional relevance and information compression. In principle, this condition provides a criterion for selecting an appropriate time scale at which CFG should be activated.
\subsection{Extension and limitation on Flow Models}

Our information bottleneck analysis can be extended to a class of stochastic flow/interpolation models.
In particular, consider
\begin{align}\label{eq:flow}
    X_t=\alpha_t X_0+\sigma_t X_1,\qquad t\in[0,1],
\end{align}
where \(X_0\) denotes the target distribution, \(X_1\) denotes the prior distribution, and
\(\alpha_0=\sigma_1=1,\ \alpha_1=\sigma_0=0\).
Assume that \(X_1\sim\mathcal N(0,I)\) is independent of \(X_0\).

\begin{theorem}[Information Bottleneck Optimality for Flow Models] \label{thm:flow_ib}
Assume that $X_t$ follows the forward SDE \eqref{eq:flow}, and that the mutual information in \eqref{eq CFG IB} are differentiable with respect to $t$. Assume that by choosing an appropriate $\beta$, the minimizer $t^*$ of \eqref{eq CFG IB} lies in the interior of the interval $(0, 1)$.Then $t^*$ satisfies
\begin{align}
    \beta J(X_{t^*}; C)=J(X_{t^*}; X_0).
\end{align}
Equivalently,
\begin{align}
    \beta\left(J(X_{t^*}\mid C)-J(X_{t^*})\right)
    =
    J(X_{t^*}\mid X_0)-J(X_{t^*}).
\end{align}
\end{theorem}

\begin{proof}
    See Appendix~\ref{proof2}.
\end{proof}

Since the conditional distribution is 
$X_t|X_0 \sim \mathcal N(\alpha_t X_0,\sigma_t^2 I),$
we have $J(X_t|X_0)=\frac{n}{\sigma_t^2}.$
Therefore, the optimality condition at \(t^*\) becomes
\begin{align}\label{eq:IBDiff_flow}
    \beta \Big(\mathbb E[\|s_\theta(x_{t^*},t^*,c)\|^2]-\mathbb E[\|s_\theta(x_{t^*},t^*,\varnothing)\|^2]\Big)
    =
    \frac{n}{\sigma_{t^*}^2}
    -
    \mathbb E[\|s_\theta(x_{t^*},t^*,\varnothing)\|^2].
\end{align}

\paragraph{Limitation.}
Theorem~\ref{thm:flow_ib} applies only to \emph{stochastic} interpolation paths with an independent Gaussian endpoint.
It does not directly cover deterministic couplings, such as $n$-rectified flow \cite{Liu2022FlowSA} for $n\ge2$, where \(X_t\) is a deterministic function of \(X_0\) and the standard continuous information bottleneck may degenerate.
Nevertheless, \eqref{eq:IBDiff_flow} can still be used as a practical heuristic in such models, although without the same rigorous information-theoretic justification. 
As for how to estimate score on Rectified Flow, a detailed discussion is provided in Appendix~\ref{appen:discuss}.

\subsection{Justification via Hyperparameter Rescale}
\begin{figure*}[!hbt]
    \centering
    \includegraphics[width=\linewidth]{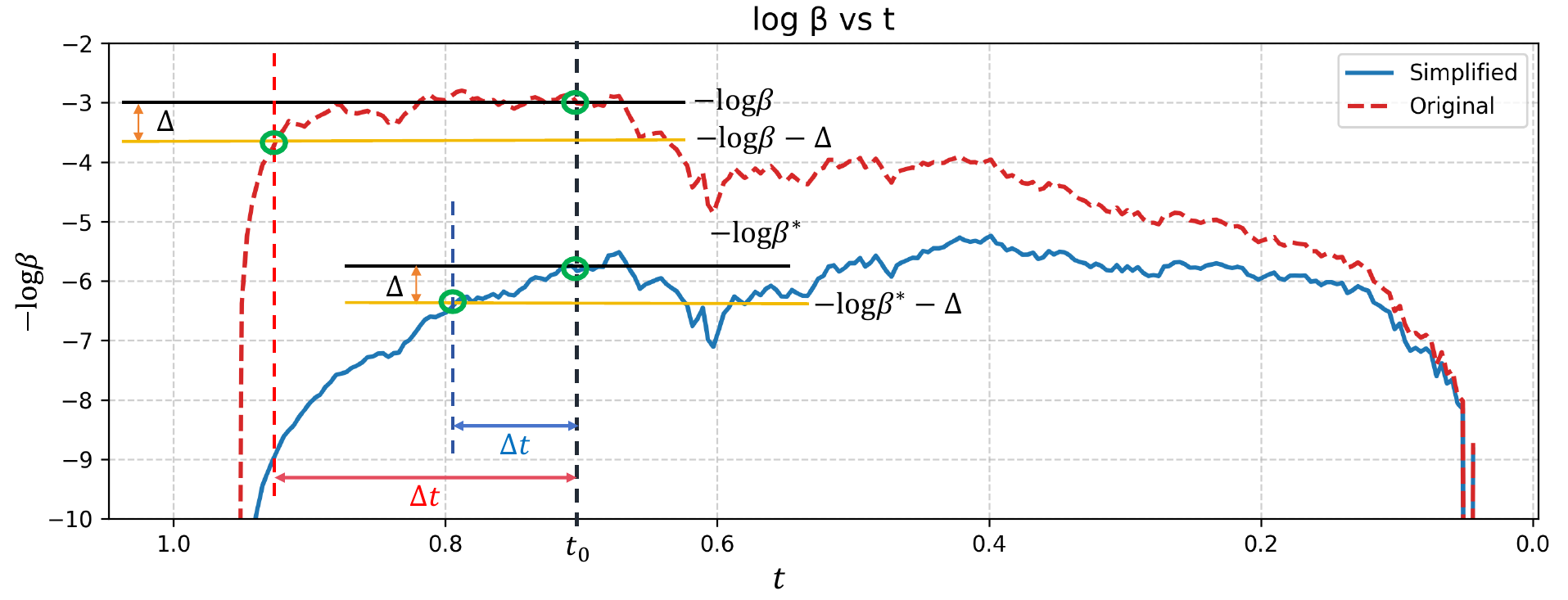}
    \caption{Sensitivity of the solution $t^*$ to the hyperparameter in the original and simplified formulations.
The original parameterization (\textcolor{red}{red $\Delta t$}) exhibits a much flatter $\log \beta$--$t$ relationship than the rescaled version (\textcolor{blue}{blue $\Delta t$}), indicating higher sensitivity to perturbations. Notably, for a fixed perturbation in $\log\beta$, the resulting deviation in 
$t$ ($\Delta t$) is significantly larger in the original formulation. Experiments are conducted on SiT-XL/2.}
        \label{fig:solution}
\end{figure*}
Despite its conceptual appeal, directly solving \eqref{eq:IBDiff} in practice is challenging. The expectations in \eqref{eq:IBDiff} are estimated using finite mini-batches during sampling, and involve squared score magnitudes, which are known to exhibit high variance. As a result, the empirical estimates of the Fisher information terms can be noisy and biased, especially at early diffusion times where score norms vary rapidly.
To mitigate this issue, we propose a parameter renormalization perspective to solve the problem. Consider the original bottleneck objective (\eqref{eq:original}) and the simplified objective (\eqref{eq:simplified}) defined with a surrogate scale $\beta^*$:

\begin{align}
    \text{Original: } & \quad \beta \cdot \Delta \mathcal{J}(t) = \frac{n}{m(t)} - \mathbb E[\|s_\theta(x_{t},t,\varnothing)\|^2], \label{eq:original} \\
    \text{Simplified: } & \quad \beta^* \cdot \Delta \mathcal{J}(t) = \frac{n}{m(t)}, \label{eq:simplified}
\end{align}
where $\Delta \mathcal{J}(t) = \mathbb E[\|s_\theta(x_{t},t,c)\|^2] - \mathbb E[\|s_\theta(x_{t},t,\varnothing)\|^2]$ denotes the conditional guidance gain.

We define the ratio of the marginal Fisher information to the noise kernel Fisher information as $\eta(t)$:
\begin{align}
    \eta(t) := \frac{\mathbb E[\|s_\theta(x_{t},t,\varnothing)\|^2]}{n/m(t)} = \frac{J(X_t)}{J(X_t|X_0)}.
\end{align}
By substituting $\eta(t)$ into ~\eqref{eq:original}, we can rewrite the original equilibrium condition as:
\begin{align}
    \beta \cdot \Delta \mathcal{J}(t) = \frac{n}{m(t)} (1 - \eta(t)) 
    \iff \left( \frac{\beta}{1 - \eta(t)} \right) \cdot \Delta \mathcal{J}(t) = \frac{n}{m(t)}.
\end{align}
Comparing this with ~\eqref{eq:simplified}, we establish a direct mapping between the original hyperparameter $\beta$ and the simplified hyperparameter $\beta^*$:
\begin{align}
    \beta^* = \frac{\beta}{1 - \eta(t)}.
    \label{eq:beta_mapping}
\end{align}

To further analyze the stability of the two formulations, we examine the sensitivity of the solution $t^*$ with respect to perturbations in the hyperparameter.
Figure~\ref{fig:solution} shows the relationship between $\log \beta$ and the corresponding solution $t$ for the original objective, as well as $\log \beta^*$ for the simplified objective. We observe that the original parameterization exhibits a much flatter slope, which suggests that the inverse mapping from $\beta$ to $t^*$ is poorly conditioned, making the root-finding procedure unstable in practice. 
In contrast, the simplified formulation yields a significantly steeper curve, implying improved conditioning and substantially reduced sensitivity to hyperparameter perturbations.
As a result, it enables more robust and reliable tuning in practice.

Thus, our objective shifts from identifying the exact physical $\beta$ to determining the \textit{effective} scale $\beta^*$ that satisfies the simplified bottleneck condition. Solving the simplified problem with $\beta^*$ yields a solution $t^*$ that is asymptotically consistent with the original problem under $\beta$. Since $\beta$ is already an empirical hyperparameter tuned for sample quality, absorbing the scaling factor $(1-\eta(t))^{-1}$ into the selection of $\beta$ incurs no loss of generality. This allows us to use the computationally accurate \eqref{eq:simplified} as a proxy for the exact information bottleneck.

%% file: sec/method.tex
\section{Method}
\begin{figure*}[!hbt]
    \centering
    \includegraphics[width=\linewidth]{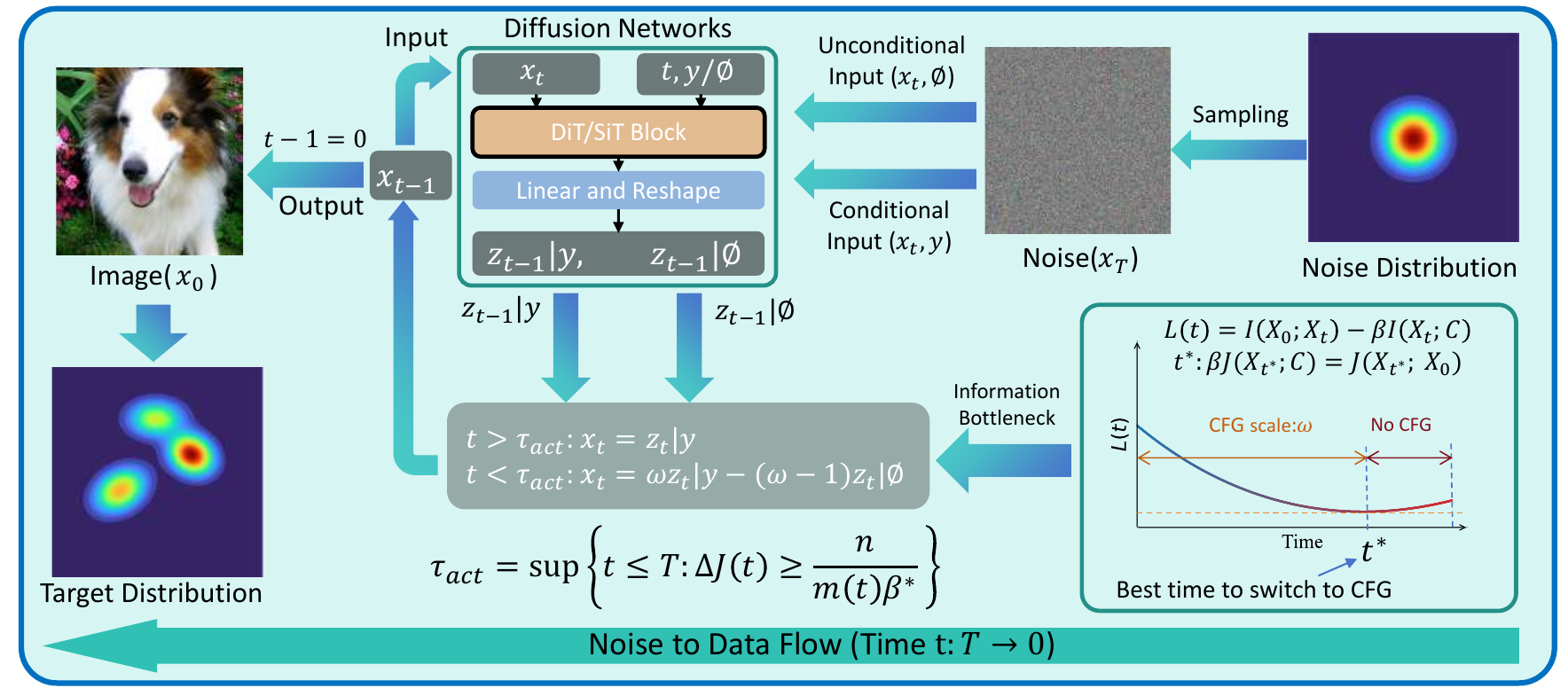}
    \caption{Overview of our adaptive CFG activation pipeline. Starting from noise, the diffusion model jointly evaluates conditional and unconditional scores at each reverse step, computes a guidance-gain statistic $\Delta \mathcal J(t)$, and activates classifier-free guidance only after the stopping criterion is met. The illustration uses a DiT/SiT-style diffusion backbone.}
        \label{fig:noise2img}
\end{figure*}

Motivated by the theoretical and empirical analysis in Section~\ref{theory}, we propose a practical procedure to determine when classifier-free guidance should be activated. Unlike fixed schedules, our method operates dynamically on a per-batch basis.
Rather than explicitly solving the equilibrium condition in continuous time, we construct a stochastic process from mini-batch statistics of the learned score function. We then activate guidance at the first time this process crosses a prescribed threshold. An overview of the proposed sampling procedure is shown in Figure~\ref{fig:noise2img}. The figure illustrates how our method augments the standard reverse diffusion process with a batch-wise activation rule: at each step, the model jointly evaluates conditional and unconditional scores, estimates the guidance gain from their batch statistics, and turns on CFG only after the estimated gain exceeds the threshold implied by our information-bottleneck analysis.

At each diffusion step $t$, we perform a single forward pass by concatenating conditional and unconditional inputs within the same mini-batch, yielding score estimates $s_\theta(x_t,t,c)$ and $s_\theta(x_t,t,\varnothing)$.
This allows us to define the batchwise estimators
\begin{align}
    \widehat{\mathcal J}_{\mathrm{cond}}(t)
    = \frac{1}{B}\sum_{i=1}^B \|s_\theta(x_t^{(i)},t,c^{(i)})\|^2, \quad 
    \widehat{\mathcal J}_{\mathrm{uncond}}(t)
    = \frac{1}{B}\sum_{i=1}^B \|s_\theta(x_t^{(i)},t,\varnothing)\|^2,
\end{align}
which serve as Monte Carlo approximations to the conditional and marginal Fisher information, respectively.
We then consider the stochastic process
\begin{align}\label{eq:Delta}
    \Delta \widehat{\mathcal J}(t)=\widehat{\mathcal J}_{\mathrm{cond}}(t)-\widehat{\mathcal J}_{\mathrm{uncond}}(t),
\end{align}

which acts as a proxy for the instantaneous guidance gain predicted by the information bottleneck condition.
Since the generative process of diffusion models proceeds \emph{backward in time} from the noise prior at $T$ to the data distribution at $0$, it is natural to perform the threshold check in the same reverse temporal order. Specifically, we perform a backward scan from the final diffusion step $T$ to $0$ and select the \emph{latest} time step (i.e., the first encountered in the reverse loop) where the guidance gain exceeds the threshold.
Formally, this defines a \emph{backward stopping time}
\[
\tau_{\mathrm{activate}}
=
\sup\left\{
t \le T:\;
 \Delta \widehat{\mathcal J}(t) \ge \frac{n}{m(t)\beta^*}
\right\}.
\]

Guidance is then activated for all subsequent steps in the forward progression of the generative process (i.e., for all $t \le \tau_{\mathrm{activate}}$). 
Algorithm~\ref{alg:cfg-backward} outlines our adaptive guidance activation procedure.

\begin{tcolorbox}[
  enhanced,
  breakable,
  colback=gray!8,
  colframe=black!70,
  boxrule=1pt,
  arc=3pt,
  left=8pt, right=8pt, top=6pt, bottom=6pt,
]
\noindent\textbf{\textcolor{orange!80!red}{Discussion:} Why activate CFG near the information bottleneck?}

\medskip
Activating classifier-free guidance (CFG) near the information bottleneck balances data fidelity and conditional relevance.  The information bottleneck objective
\[
  L(t) = I(X_0;\, X_t) - \beta\, I(X_t;\, C)
\]
captures the balance between retaining information from the data distribution
($I(X_0; X_t)$) and incorporating information from the conditioning variable
($I(X_t; C)$). The optimal time for guidance activation corresponds to the point
where this balance reaches a critical threshold.

\medskip
In the early stages of the reverse diffusion process ($t \approx T$), the sample is
dominated by noise, and both $I(X_0; X_t)$ and $I(X_t; C)$ are small. At this point,
applying CFG would amplify noise rather than useful structure.
As the process progresses, $I(X_0; X_t)$ increases, while $I(X_t; C)$ remains weak.
This causes the guidance signal to be less effective. However,
once $I(X_t; C)$ starts to dominate, $L(t)$ reaches a minimum at a critical point $t^*$. Here, data fidelity and conditional information are optimally balanced for guidance.

\medskip
Consequently, the guidance term $\omega\,(\epsilon_\theta(x_t, t, c) - \epsilon_\theta(x_t, t, \varnothing))$
in~\eqref{CFG} acts as a control signal that is most effective at $t^*$. Activating CFG at the minimum of
$L(t)$ ensures the model leverages the conditional signal reliably without compromising the underlying data structure.
\end{tcolorbox}
\input{tabel/algorithm}
\paragraph{Relationship to Interval Guidance \cite{kynkaanniemi2024applying}.}
Our framework recovers Interval Guidance in the large-batch regime. As the batch size grows to $\infty$, the mini-batch score statistics in our thresholding rule concentrate, causing the random activation time $\tau_{\mathrm{activate}}$ to approach a deterministic boundary common to all samples. This fixed boundary can be identified with the activation endpoint $t_h$ used in Interval Guidance. Compared with Interval Guidance, however, our method does not require manually prescribing this boundary in finite-batch settings; instead, it infers the activation time adaptively from mini-batch score statistics, and the resulting rule is grounded in our information-bottleneck analysis. 
Furthermore, we explore the statistics properties of $\tau_{\mathrm{activate}}$ in Appendix \ref{appen:exp}.

%% file: tabel/algorithm.tex
\begin{center} 
    \scalebox{1.0}{%
        \begin{minipage}{1.0\linewidth} 
            \begin{algorithm}[H]
    \caption{Reverse Diffusion with Adaptive Guidance Activation}
    \label{alg:cfg-backward}
    \begin{algorithmic}[1]
        \Require $\{\x_T^{(i)}\}_{i=1}^B \stackrel{\text{i.i.d.}}{\sim} \mathcal{N}(0, \rmI_n)$, $\omega, \beta^* > 0$, $\tau_{\rm back} = \text{undefined}$.
        \For{$t = T$ {\bfseries to} $1$}
            \State Calculate noise predictions $\epscond(\x_t^{(i)}), \epsnullnc(\x_t^{(i)})$.
            \State $\frac{m(t)}{n}\Delta \widehat{\mathcal J}(t) \gets \frac{1}{Bn}\sum_{i=1}^B \left(\|\epscond(x_t^{(i)})\|^2 - \|\epsnullnc(x_t^{(i)})\|^2\right)$ \textcolor{blue}{\Comment{ Compute for current batch}}
            
            \If{$\tau_{\rm back}$ is undefined}
                \If{$\Delta \widehat{\mathcal J}(t) \ge \frac{n}{\beta^* m(t)}$}
                    \State $\tau_{\rm back} \gets t$
                \EndIf
            \EndIf
            
            \If{$\tau_{\rm back}$ is defined}
                \State $\epscfgnc(\x_t) \gets \epsnullnc(\x_t) + \omega (\epscond(\x_t) - \epsnullnc(\x_t))$, update $x_{t-1}$
            \Else
                \State $\epscfgnc(\x_t) \gets \epscond(\x_t)$, update $x_{t-1}$
            \EndIf
        \EndFor
        \State {\bfseries return} $x_0$
    \end{algorithmic}
\end{algorithm}
        \end{minipage}
    }
\end{center}

%% file: sec/result.tex
\section{Results}\label{experiments results}
\subsection{Experimental Setup}
\noindent \textbf{Datasets and Models.}
We evaluate our method on class-conditional image generation and text-to-image synthesis. We conduct experiments on ImageNet \citep{5206848} and MS-COCO \citep{lin2014microsoft}. All models are based on advanced diffusion  formula, including Flow Matching \cite{Lipman2022FlowMF} (SiT \citep{ma2024sit, yu2025repa}), DDPM \cite{Ho2020DenoisingDP} (DiT \citep{Peebles2022DiT}) and Rectified Flow \cite{Liu2022FlowSA} (Stable Diffusion \citep{rombach2022high}, FLUX \citep{labs2025flux1kontextflowmatching}), using official pre-trained weights.   

\noindent \textbf{Evaluation Metrics and Implementation Details.}
We use FID \citep{heusel2017gans}, sFID \cite{ding2022continuous} and IS \citep{salimans2016improved} to assess both fidelity and conditional alignment. We use FID \citep{heusel2017gans}, Clip score \citep{Radford2021LearningTV}, HPSv2 \citep{wu2023human} and Aesthetic score \cite{improvedaestheticpredictor} for text-to-image tasks.
All experiments are run on NVIDIA GPUs. 
\subsection{Experimental Results}

\input{tabel/merge}
\noindent \textbf{Results on Class-Conditional Tasks.}
In Table \ref{tab:fid}, we evaluate our method on both DiT-XL/2 and SiT-XL/2 models, tested on the ImageNet benchmarks (256 and 512 resolutions) for class-conditional image generation. Our method demonstrates consistent improvements across all metrics when applied to both DiT and SiT models, especially in FID and IS scores. 

For DiT-XL/2, our method outperforms the baseline across both resolutions, with a notable FID score of \textbf{2.06} on ImageNet-256 and \textbf{6.49} on ImageNet-512, showing strong performance. Similarly, for SiT-XL/2, our method yields the best FID score of \textbf{1.41} for SDE sampler and \textbf{1.47} for ODE sampler, surpassing the baseline and achieving a significant boost in IS scores.
These results highlight the versatility and effectiveness of our approach, which consistently enhances the performance of both DiT and SiT models across different resolutions and datasets.

\input{tabel/T2I}
\noindent\textbf{Results on T2I Tasks.}
As shown in Table~\ref{tab:different_guidance_compare},
our method achieves consistent improvements
over standard CFG on SD3 in terms of FID,
HPSv2.1, CLIP Score, and Aesthetic Score.
On FLUX.1 at $256\times256$ resolution,
our method improves FID and HPSv2.1,
with a slight reduction in CLIP Score.

We further evaluate FLUX.1 at
$1024\times1024$ resolution with
$\omega=7.5$. Our method achieves an
average HPSv2.1 score of 30.73 with
$B=16$, compared with 20.13 for
standard CFG. Importantly, it also
achieves 30.29 under single-sample
inference ($B=1$), demonstrating
its applicability without requiring
multiple generated samples.

\noindent \textbf{Ablation Study.}
As shown in Table~\ref{tab:ablation}, we investigate
the impact of batch size on generation performance.
Our method consistently outperforms the baseline
across all tested batch sizes
($B\in\{4,8,16,32\}$), achieving lower FID scores
in every setting. To further assess the
reproducibility of these improvements, we conduct
multi-seed experiments. As reported in
Table~\ref{tab:multi_seed}, our method achieves
an FID of $4.98\pm0.06$ and an IS of
$298.1\pm2.6$, consistently outperforming
Interval Guidance across three random seeds.
We further analyze the distribution of the
batch-wise activation time $\tau$ under different
batch sizes. As shown in
Table~\ref{tab:stopping_batch}, the mean activation
time remains stable at approximately $0.68$,
while its standard deviation decreases as the
batch size increases. This observation is
consistent with the concentration of batch-level
statistics and supports the connection between
our adaptive criterion and fixed-interval guidance
in the large-batch regime.
Finally, we compare our Fisher-derived
score-norm difference with an alternative
activation statistic based on the
conditional--unconditional score-discrepancy MSE.
As shown in Table~\ref{tab:statistic_ablation},
our criterion achieves better FID and IS under
the tested configurations, suggesting its
practical effectiveness in identifying a
suitable guidance activation point.
Additional results are provided in
Appendix~\ref{appen:exp}.

\begin{table*}[t]
\centering

\begin{minipage}[t]{0.40\textwidth}
    \vspace{0pt}
    \centering
    \caption{\textbf{Multi-seed stability.}
    FID and IS over three seeds on SiT-XL/2.}
    \label{tab:multi_seed}
    \vspace{2pt}
    \scriptsize
    \setlength{\tabcolsep}{2pt}
    \resizebox{\linewidth}{!}{%
    \begin{tabular}{llcccc}
        \toprule
        Method & Metric & Seed 0 & Seed 1
        & Seed 2 & Mean $\pm$ Std. \\
        \midrule
        \multirow{2}{*}{Interval}
        & FID$\downarrow$
        & 5.09 & 5.25 & 5.19 & $5.18\pm0.08$ \\
        & IS$\uparrow$
        & 287.4 & 280.9 & 284.1 & $284.1\pm3.3$ \\
        \midrule
        \multirow{2}{*}{Ours}
        & FID$\downarrow$
        & 4.92 & 5.05 & 4.98
        & $\mathbf{4.98\pm0.06}$ \\
        & IS$\uparrow$
        & 299.7 & 295.1 & 299.5
        & $\mathbf{298.1\pm2.6}$ \\
        \bottomrule
    \end{tabular}%
    }
\end{minipage}
\hfill
\begin{minipage}[t]{0.24\textwidth}
    \vspace{0pt}
    \centering
    \caption{\textbf{Stopping-time statistics.}
    Activation time across batch sizes.}
    \label{tab:stopping_batch}
    \vspace{1.5pt}
    \scriptsize
    \setlength{\tabcolsep}{3pt}
    \resizebox{\linewidth}{!}{%
    \begin{tabular}{ccc}
        \toprule
        Batch & Mean ($\tau$) & Std. ($\tau$) \\
        \midrule
        $B=8$  & 0.68 & 0.107 \\
        $B=16$ & 0.68 & 0.086 \\
        $B=32$ & 0.68 & 0.054 \\
        \bottomrule
    \end{tabular}%
    }
\end{minipage}
\hfill
\begin{minipage}[t]{0.34\textwidth}
    \vspace{0pt}
    \centering
    \caption{\textbf{Activation statistic ablation.}
    Comparison of switching criteria.}
    \label{tab:statistic_ablation}
    \vspace{2pt}
    \scriptsize
    \setlength{\tabcolsep}{3pt}
    \resizebox{\linewidth}{!}{%
    \begin{tabular}{lcc}
        \toprule
        Activation Statistic
        & FID$\downarrow$ & IS$\uparrow$ \\
        \midrule
        Baseline & 5.09 & 287.4 \\
        Score-discrepancy MSE & 5.14 & 287.5 \\
        Fisher norm difference (Ours)
        & \textbf{4.92} & \textbf{299.7} \\
        \bottomrule
    \end{tabular}%
    }
\end{minipage}

\end{table*}

\noindent \textbf{Qualitative Comparison.} 
We present visual comparisons of text-to-image generation results using Stable Diffusion 3 and Flux.1 in Figure \ref{fig:qualitive_multi}. The red text in the prompts indicates specific semantic details that the baseline model fails to generate correctly. The regions highlighted by boxes illustrate the differences between the methods. As shown, our approach successfully aligns with these complex semantic requirements, generating details that are missing or incorrect in the baseline results. 

\begin{figure*}[!htb]
    \centering
    \includegraphics[width=\linewidth]{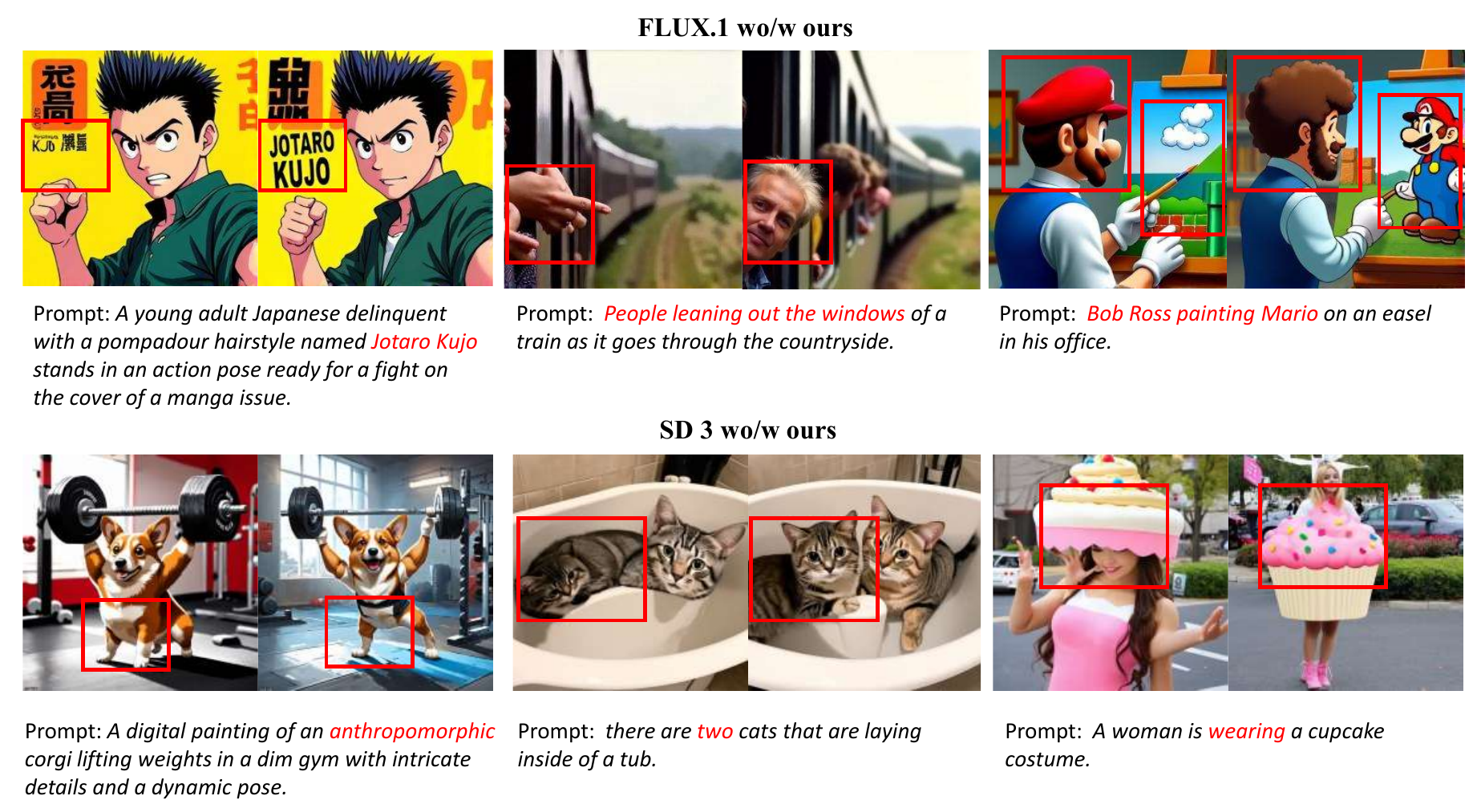}
    
    \caption{\textbf{Qualitative Comparison.} Visual results on text-to-image tasks using Stable Diffusion 3 and Flux.1. The \textcolor{red}{red text} in the prompts highlights specific details that the baseline fails to generate.}
    \label{fig:qualitive_multi}
\end{figure*}

%% file: tabel/merge.tex
\begin{table}[!hbt]
\centering
\begin{minipage}{0.62\textwidth}
\centering
\caption{\textbf{Quantitative comparison on Class-Conditional tasks.} Comparison of different evaluation metrics on Class-Conditional ImageNet datasets with different diffusion architectures. \textbf{Bold} indicates the best result, and \underline{underlined} indicates the second best.}
\label{tab:fid}
\vspace{-0mm}
\scalebox{0.69}{
\begin{tabular}{lcccccc}
\toprule
\multicolumn{7}{l}{{\bf\normalsize ImageNet (Default $B=32$}) } \\
\toprule
\multirow{1}{*}{\bf{Model-256, 50k samples, 250 inference timesteps}} 
& FID$\downarrow$ & IS$\uparrow$ & sFID$\downarrow$  \\
\arrayrulecolor{black}\midrule
DiT-XL/2 ($\omega=1.5$, SDE sampler)     & \underline{2.29} & 276.8 & 4.6  \\
DiT-XL/2 ($\omega=1.6$)     & 2.71 & \textbf{303.2} & 4.5  \\
\textbf{DiT-XL/2 + Ours}($\omega= 1.6, \log \beta^*=8.0$)       & \textbf{2.06} & \underline{280.9} & \textbf{4.5}  \\
\arrayrulecolor{gray}\cmidrule(lr){1-7}
SiT-XL/2 (REPA)($\omega=1.35$, SDE)     & 1.80 & 284.0 & \textbf{4.5} \\
SiT-XL/2 (REPA, Interval) ($\omega=1.8, t_l=0,t_h=0.7$)     & \underline{1.42} & 305.7 &\underline{4.7} \\
SiT-XL/2 (REPA) ($\omega=1.9$)     & {4.67} & \textbf{388.3} & 5.5   \\
\textbf{SiT-XL/2 (REPA) + Ours} ($\omega=1.9, \log \beta^*=5.9,B=48$)     & \textbf{1.41} & \underline{306.6} & \underline{4.7}   \\
\arrayrulecolor{gray}\cmidrule(lr){1-7}
SiT-XL/2 (REPA)($\omega=1.8, t_h=0.7$, ODE)     & 1.56 & 283.0 & 4.6 \\
\textbf{SiT-XL/2 (REPA) + Ours} ($\omega=1.8, \log \beta^*=6.5$, ODE)     & \textbf{1.47} & \textbf{291.2} & \textbf{4.4}\\
\arrayrulecolor{black}\midrule
\multirow{1}{*}{\bf{Model-512, 10k samples, 100 inference timesteps}}
\\
\arrayrulecolor{black}\midrule
DiT-XL/2 ($\omega=1.5$, SDE)     & \underline{6.81} & 229.5 & \underline{20.0} \\
DiT-XL/2 ($\omega=1.7$)     & 7.08 & \textbf{275.5} & 19.2 \\
\textbf{DiT-XL/2 + Ours}($\omega= 1.7, \log \beta^*=8.5, B=16$)       & \textbf{6.49} & \underline{241.4} & \underline{20.0} \\
\arrayrulecolor{black}\bottomrule
\end{tabular}
}
\end{minipage}%
\hfill
\begin{minipage}{0.36\textwidth}
\centering
\caption{\textbf{Ablation comparison.} Comparison of different evaluation metrics on Class-Conditional ImageNet datasets with different batch size $B$. Unless otherwise specified, all other settings are identical to those in Table \ref{tab:fid}.}


\label{tab:ablation}
\vspace{-0mm}
\scalebox{0.97}{
\begin{tabular}{lcccccc}
\toprule
\multicolumn{7}{l}{\textbf{\normalsize ImageNet-256}} \\
\multicolumn{7}{l}{\textbf{\normalsize 10k samples,}} \\
\multicolumn{7}{l}{\textbf{\normalsize 50 inference timesteps}} \\
\toprule
\multirow{1}{*}{\bf{SiT-XL/2}} 
& FID$\downarrow$ & IS$\uparrow$ \\
\arrayrulecolor{black}\midrule
Baseline (SDE)    & 5.09 & 287.4 \\
\arrayrulecolor{gray}\cmidrule(lr){1-7}
Ours($B=\textcolor{red}{32}$)     & 4.92 & 299.7   \\
Ours($B=\textcolor{red}{16}$)     & 4.89 & 296.1   \\
Ours($B=\textcolor{red}{8}$)     & 4.96 & 294.6   \\
Ours($B=\textcolor{red}{4}$)     & 4.95 & 294.9   \\
\arrayrulecolor{black}\bottomrule
\end{tabular}
}
\end{minipage}
\end{table}

%% file: tabel/T2I.tex
\begin{table*}[!ht]
    \centering
    \footnotesize
    \setlength{\tabcolsep}{1pt}
    \caption{\textbf{Quantitative comparison on text-to-image tasks.}
    Quantitative results on SD3 and FLUX.1 under different
    guidance configurations. We report FID, HPSv2.1,
    CLIP Score, and Aesthetic Score where available.
    For FLUX.1 at $1024\times1024$, we additionally
    evaluate single-sample inference ($B=1$).
    ``--'' denotes unavailable results.}
    \resizebox{\textwidth}{!}{%
    \begin{tabular}{llcccccccc}
        \toprule
        \multirow{2}{*}{Model}
        & \multirow{2}{*}{Approach}
        & \multirow{2}{*}{FID$\downarrow$}
        & \multicolumn{5}{c}{HPS v2.1}
        & \multirow{2}{*}{CLIP$\uparrow$}
        & \multirow{2}{*}{Aesthetic$\uparrow$} \\
        &&
        & paintings$\uparrow$
        & photo$\uparrow$
        & concept-art$\uparrow$
        & anime$\uparrow$
        & average$\uparrow$
        && \\
        \midrule
        \multirow{2}{*}{SD 3 (1024, 28 steps)}
        & CFG($\omega=7.5, B=4$)
        & 25.98 & 30.99 & 28.39 & 30.63 & 31.31
        & 30.33 & 31.83 & 5.44 \\
        & Interval($\omega=7.5,\ t_h=0.7T$)
        & -- & 27.66 & 24.76 & 27.57 & 28.55
        & 27.13 & 31.41 & 5.21 \\
        & Ours($\omega=7.5,\ \log\beta^*=10.0$)
        & \textbf{24.88} & \textbf{31.12}
        & \textbf{28.57} & \textbf{30.98}
        & \textbf{31.51} & \textbf{30.55}
        & \textbf{31.89} & \textbf{5.45} \\
        \midrule
        \multirow{2}{*}{FLUX.1 (256, 50 steps)}
        & CFG($\omega=1.5, B=64$)
        & 30.93 & 29.77 & 27.18 & 29.21 & 30.80
        & 29.24 & \textbf{31.10} & 5.46 \\
        & Ours($\omega=1.5,\ \log\beta^*=8.0$)
        & \textbf{30.36} & \textbf{29.87}
        & \textbf{27.38} & \textbf{29.26}
        & \textbf{30.94} & \textbf{29.36}
        & 30.93 & \textbf{5.48} \\
        \midrule
        \multirow{3}{*}{FLUX.1 (1024, 28 steps)}
        & CFG($\omega=7.5$)
        & -- & 20.41 & 18.80 & 20.71 & 20.60
        & 20.13 & -- & -- \\
        & Ours($\omega=7.5,\ \log\beta^*=8.0,\ B=16$)
        & -- & \textbf{30.84} & \textbf{29.49}
        & \textbf{30.55} & \textbf{32.05}
        & \textbf{30.73} & -- & -- \\
        & Ours($\omega=7.5,\ \log\beta^*=8.0,\ B=1$)
        & -- & 30.36 & 29.08 & 30.21 & 31.50
        & 30.29 & -- & -- \\
        \bottomrule
    \end{tabular}%
    }
    \label{tab:different_guidance_compare}
\end{table*}

%% file: sec/conclusion.tex
\section{Conclusion and Future Work}\label{conclusion}
In this work, we study Classifier-Free Guidance (CFG) from Information Bottleneck perspective. We theoretically established that conditional information is not uniformly effective throughout the generation process but becomes critical only after a specific stage characterized by a Fisher-information balance. Based on this insight, we proposed a training-free, dynamic CFG strategy that adaptively activates guidance by monitoring score statistics. Extensive experiments demonstrate that our method consistently improves sample quality and robustness across diverse architectures.

There remain several promising directions for future research. \textbf{Theoretically}, as mentioned in Section \ref{theory}, while our Information Bottleneck analysis extends to stochastic flow models, extending the framework to fully characterize deterministic couplings would be an important next step toward a more complete understanding. 
\textbf{Practically}, while our method supports
single-sample inference, batch averaging
can improve the stability of the activation
statistic. Developing more robust
sample-specific criteria and extending
the framework to heterogeneous
conditioning scenarios remain promising
directions for future research.

\noindent{\textbf{Acknowledgment.}} 
This work is supported by National Nature Science Foundation of China under grant 62671375, U24A20220, and by the Shanghai Municipal Commission of Economy and Informatization, under Grant No. 2024-GZL-RGZN-01008.

%% file: sec/appendix.tex

\section{Extended Related Work}\label{Related Work}

Conditional control in diffusion models is commonly achieved through classifier guidance (CG, \cite{Dhariwal2021DiffusionMB}) and classifier-free guidance (CFG, \cite{Ho2022ClassifierFreeDG}). CG introduces an external classifier to steer generation toward the target condition:
\begin{align}\label{CG}
    \hat{\epsilon}(x_t, t, y) = \epsilon_\theta(x_t, t, \varnothing) + \omega \, \Sigma_\theta(x_t, t) \nabla \log p_t(y \mid x_t),
\end{align}
while CFG replaces the classifier with the difference between conditional and unconditional score estimates:
\begin{align}
    \hat{\epsilon}(x_t, t, y) = \omega \left[ \epsilon_\theta(x_t, t, y) - \epsilon_\theta(x_t, t, \varnothing) \right] + \epsilon_\theta(x_t, t, \varnothing).
\end{align}
In both cases, the guidance scale $\omega$ controls the trade-off between fidelity and diversity. Owing to its simplicity and strong empirical performance, CFG has become the dominant guidance strategy in modern diffusion models [4]. However, most existing CFG-based methods apply a fixed guidance rule uniformly across all reverse diffusion steps, and recent studies have shown that this design is often suboptimal.

A broad series of work therefore investigates dynamic or adaptive guidance \emph{strength}. For example, \citep{malarz2025classifier} proposes a $\beta$-adaptive strategy that enforces weak guidance in the early and late stages while strengthening it in the middle phase of denoising. \citep{Sadat2025GuidanceIT} introduces Frequency-Decoupled Guidance (FDG), which applies different guidance strengths to low- and high-frequency components. \citep{lin2024common} rescales CFG to alleviate over-exposure effects, while \citep{poleski2025geoguide} proposes a geometry-aware normalization strategy to maintain stable guidance magnitude during late denoising. These methods suggest that guidance should vary with time or signal structure, and mainly focus on designing better weighting functions $w(t)$ or modified guidance parameterizations.

To mitigate the computational cost of training distinct conditional and unconditional models, several approaches explore the idea of \textit{degrading} the conditional network to mimic unconditionality.
\cite{Karras2024autoguidance} achieves this by utilizing an under-trained state of the network, arguing that early-stopped weights contain less semantic bias and thus behave more unconditionally. 
In a complementary vein, \cite{rajabi2025token} introduces stochasticity by adding noise to the conditioning tokens directly. 
By feeding corrupted inputs into the conditional network, the output becomes a degraded approximation of the true unconditional distribution, effectively replacing the need for a separate network.

As for when to activate CFG, a representative example is Interval Guidance \citep{kynkaanniemi2024applying}, which shows that applying CFG over all denoising steps is suboptimal: guidance tends to harm diversity in the high-noise regime, has limited effect in the very low-noise regime, and is most beneficial only over an intermediate interval. Their method therefore restricts guidance to a fixed range of timesteps, improving both sample quality and distributional fidelity. While this line of work highlights the importance of \emph{when} guidance is applied, existing methods still largely rely on hand-designed schedules or fixed activation intervals.

Beyond schedule-based reweighting, another series of work revisits the guidance mechanism itself. CFG-Ctrl \cite{wang2026cfg} formulates the denoising process as a dynamical system and interprets most existing CFG-based methods from a control-theoretic perspective, arguing that purely linear guidance may become unstable at large guidance scales and proposing a sliding-mode-control-based variant to improve robustness. In a different direction, Momentum Guidance \cite{liao2026momentum} modifies the sampling dynamics by incorporating momentum into the CFG-adjusted flow update, providing a training-free and plug-and-play way to improve sampling robustness without retraining the model. Self-Guidance \cite{li2025self} instead derives an additional training-free guidance signal from temporal prediction discrepancies within the model itself, comparing predictions at the current and shifted timesteps and combining the resulting correction with CFG.

On a more theoretical front, several recent works have begun to analyze the role of conditioning in diffusion models through the lens of information theory.  \cite{11195367} proposes an information-theoretic approach that maximizes the conditional mutual information to ensure the generation process retains maximum information about the measurements. Concurrently, \cite{yuan2024dynamic} proposes the Dynamic Graph Information Bottleneck (DGIB) to learn robust and discriminative representations, defending against attacks on link prediction tasks. These studies align closely with our goal of adaptive activation, as both seek to identify the most informative moments for conditioning.

Recent studies have explored information-theoretic
principles to better understand and improve
diffusion-based generative modeling.
\cite{shen2025information,DBLP:conf/aaai/ZhengGJYWZCWL26}
generalize the connection between KL divergence
,Fisher divergence and Wasserstein distance to non-Gaussian isotropic
noise perturbations, introducing a warm-started,
variance-aware framework for tighter likelihood
estimation and more stable diffusion training.
More recently, Shen et al.~\cite{shen2026rnyi}
propose Rényi Diffusion Models, which formulate
diffusion training through Rényi divergence
and derive a generalized score-matching
objective to explicitly control the trade-off
between sample quality and distribution coverage.
While these studies primarily investigate
information-theoretic objectives for learning
generative distributions, our work focuses
on the role of conditional information during
inference. Specifically, we formulate the
diffusion trajectory as a time-dependent
information bottleneck and derive a
Fisher-information balance condition
to characterize an effective region for
CFG activation. This leads to a training-free
guidance strategy that dynamically determines
when conditional guidance should be applied.

Overall, these studies demonstrate that diffusion guidance can be improved by adjusting its strength, modifying its direction, or altering how guided updates are integrated over time. In contrast, our work focuses on \emph{adaptive guidance activation}: rather than prescribing a fixed schedule or interval in advance, we study when conditional guidance should be activated based on the generation dynamics themselves. Notably, our formulation also provides a unified perspective on interval-based guidance, which can be interpreted as a special case under a limiting setting of our adaptive activation rule.


\section{Proof of Theorems}\label{proof}
\subsection{Proof of Lemmas}\label{proof0}
In this section, we present several auxiliary results that serve as
preparatory steps for the proof of Theorem~\ref{thm:cfg_ib} and \ref{thm:flow_ib}.
First, we establish a general entropy evolution formula for diffusion processes.
Second, we specialize this result to mutual information under conditioning. 
\begin{lemma}\label{entropy eq}
    Let $X_t\in \mathbb R^n$ be the solution of the SDE 
    \begin{align}\label{SDE_IT}
        \rd X_t = f(X_t,t)\rd t+g(t)\rd w_t.
    \end{align}
    
    Then
    \begin{align}
        \frac{\rd}{\rd t}H(X_t) &= \frac{1}{2}g^2(t) J(X_t)+\mathbb E_{p_t}[\nabla \cdot f_t],
    \end{align}
    where $p_t$ is the marginal distribution of $X_t$. In particular, for the Ornstein--Uhlenbeck (OU) process
$$\rd X_t=f(t) X_t \rd t+g(t) \rd w_t,$$ we have the form:
    \begin{align}
        \frac{\rd}{\rd t}H(X_t) &= \frac{1}{2}g^2(t)J(X_t)+\mathbb E[\nabla\cdot (f(t)X_t)]\\
        &= \frac{1}{2}g^2(t)J(X_t)+ nf(t)
    \end{align}
\end{lemma}
\begin{proof}
    FPE of \eqref{SDE_IT} is 
    \begin{align*}
        \partial_tp_t = -\nabla\cdot(f_tp_t)+\frac{1}{2}g^2(t)\Delta p_t,
    \end{align*}
    as for differential entropy $H(X_t)=-\int p_t\log p_t\rd x$, we obtain
    \begin{align*}
        \frac{\rd}{\rd t}H(X_t) &= -\int \partial_t p_t \log p_t\rd x - \int \partial_t p_t \rd x\\
        &= -\int \partial_t p_t \log p_t\rd x\\
        &= \int \nabla\cdot(f_tp_t) \log p_t\rd x - \int \frac{1}{2}g^2(t)\Delta p_t \log p_t\rd x,
    \end{align*}

    then we calculate the terms in the above equation,
    \begin{align*}
        \int \nabla\cdot(f_tp_t) \log p_t\rd x &=
        (-1)\int\innerprod{f_tp_t}{\frac{\nabla p_t}{p_t}}\\
        &=-\int \innerprod{f_t}{\nabla p_t}\\
        &= \mathbb E_{p_t}[\nabla \cdot f_t],
    \end{align*}
    using $\Delta \log p = \Delta p/p-(\nabla \log p)^2$,
    \begin{align*}
        \int \frac{1}{2}g^2(t)\Delta p_t \log p_t\rd x &= \frac{1}{2}g^2(t)
        \int p_t\Delta\log p_t\\
        &=\frac{1}{2}g^2(t)
        \int p_t(\Delta p_t/p_t-(\nabla \log p_t)^2)\\
        &= -\frac{1}{2}g^2(t) J(X_t),
    \end{align*}
    so we obtain 
    \begin{align*}
        \frac{\rd}{\rd t}H(X_t) &= \frac{1}{2}g^2(t) J(X_t)+\mathbb E_{p_t}[\nabla \cdot f_t].
    \end{align*}
\end{proof}
Lemma~\ref{entropy eq} characterizes how differential entropy evolves along a diffusion trajectory. 
We next leverage this result to study the time evolution of mutual information under conditioning.

\begin{lemma}[Information evolution for diffusion models]\label{thm Fisher}
    We treat the conditioning variable $Y$ as a random variable, which may correspond to $C$ or $X_0$ in our setting. The output $X_t$ of the SDE \eqref{SDE_IT} with condition $Y$ satisfies 
    \begin{align}
        \frac{\rd}{\rd t}I(X_t;Y)&=-\frac{1}{2}g^2(t)J(X_t; Y).
    \end{align}
\end{lemma}
\begin{proof}
    Lemma \ref{entropy eq} holds for any initial distribution. In particular, if $y$ is the given condition of conditional diffusion models, we have
    \begin{align}
        \frac{\rd}{\rd t}H(X_t|Y=y) &= \frac{1}{2}g^2(t)J(X_t|Y=y)+ \mathbb E_{p_t(\cdot|y)}[\nabla \cdot f_t],
    \end{align}
    where $p_t(\cdot|y)$ is the marginal distribution of $X_t$ given $Y=y$. For any condition $Y$, we obtain
    \begin{align}
        \frac{\rd}{\rd t}H(X_t|Y) =\frac{\rd}{\rd t}\mathbb E_{p(y)}H(X_t|Y=y)= \frac{1}{2}g^2(t)J(X_t|Y)+ \mathbb E_{p_t}[\nabla \cdot f_t],
    \end{align}
    where $p_t(\cdot)$ is the marginal distribution of $X_t$.
    Then we conclude that
    \begin{align}
        \frac{\rd}{\rd t}I(X_t;Y)&=\frac{\rd}{\rd t}(H(X_t)-H(X_t|Y))\\
        &= \frac{1}{2}g^2(t)(J(X_t) - J(X_t|Y))\\
        & = -\frac{1}{2}g^2(t)J(X_t; Y),
    \end{align}
    as desired.
\end{proof}
We are now ready to complete the proof of Theorem~\ref{thm:cfg_ib} and \ref{thm:flow_ib} by applying the above mutual information dynamics to the information bottleneck objective.

\subsection{Proof of Theorem~\ref{thm:cfg_ib}}\label{proof1}
In this section, we provide a complete proof of Theorem~\ref{thm:cfg_ib}. 
The proof proceeds in three steps. We apply the lemmas proved in Appendix \ref{proof0} to the information bottleneck objective and derive the optimality condition.

\begin{proof}[Proof of Theorem~\ref{thm:cfg_ib}]
Let $L(t)$ denote the information bottleneck objective defined in \eqref{eq:IBDiff}.
If $t^*$ is an interior minimizer of $L(t)$, it must satisfy the first-order optimality condition
\begin{align}
    \frac{\rd L(t^*)}{\rd t^*} = 0.
\end{align}
Applying Lemma~\ref{thm Fisher} to both mutual information terms appearing in $L(t)$, we obtain
\begin{align}
    \beta J(X_{t^*};C) = J(X_{t^*};X_0),
\end{align}
which completes the proof.
\end{proof}

\subsection{Proof of Theorem~\ref{thm:flow_ib}}\label{proof2}
In this section, we provide a complete proof of Theorem~\ref{thm:flow_ib}. 
\begin{lemma}[Information evolution for flow models]\label{thm:flow_fisher}
Let the flow model be defined by
\begin{align}
    X_t=\alpha_tX_0+\sigma_tX_1,
\end{align}
where $\alpha_t>0$ and $\sigma_t\ge 0$ are differentiable functions of $t$, and $X_1\sim\mathcal N(0,I)$ is independent of $(X_0,Y)$. 
Consequently, for any conditioning variable $Y$,
\begin{align}
    \frac{\rd}{\rd t}I(X_t;Y)
    =-\frac{1}{2}g^2(t)J(X_t;Y)
    =-\left(\sigma_t\dot\sigma_t-\frac{\dot\alpha_t}{\alpha_t}\sigma_t^2\right)J(X_t;Y).
\end{align}
\end{lemma}
\begin{proof}
    The reparameterized form of OU-SDE $\rd X_t=f(t) X_t \rd t+g(t) \rd w_t$
    is 
    \begin{align}
        X_t=e^{\int_0^tf(s)\rd s} X_0 + \sqrt{\int_0^t
\exp\left(2\int_s^t f(r)\mathrm dr\right)
g^2(s)\mathrm ds} Z, 
    \end{align}
    so for flow model
    \begin{align}
        X_t=\alpha_tX_0+\sigma_t X_1, 
    \end{align}
    we can find its FPE by turning it into OU-SDE by defining
    \begin{align}
        \alpha_t = e^{\int_0^tf(s)\rd s}, ~~~~\sigma^2_t=\int_0^t
\exp\left(2\int_s^t f(r)\mathrm dr\right)
g^2(s)\mathrm ds = \alpha_t^2\int_0^t\frac{g^2(s)}{\alpha_s^2}\rd s, 
    \end{align}
    the solution can be easily found:  
    \begin{align}
        f(t)=\frac{\dot{\alpha_t}}{\alpha_t}, ~~~g^2(t)=2\sigma_t\dot{\sigma_t}-\frac{2\dot{\alpha_t}\sigma_t^2}{\alpha_t}, 
    \end{align}
    so using Lemma \ref{thm Fisher}, 
    \begin{align}
        \frac{\rd}{\rd t}I(X_t;Y)&= -\frac{1}{2}\left(2\sigma_t\dot{\sigma_t}-\frac{2\dot{\alpha_t}\sigma_t^2}{\alpha_t}\right)J(X_t; Y)\\
        &=-\left(\sigma_t\dot{\sigma_t}-\frac{\dot{\alpha_t}\sigma_t^2}{\alpha_t}\right)J(X_t; Y),
    \end{align}
\end{proof}
\begin{proof}[Proof of Theorem \ref{thm:flow_ib}]
    Thus just like the proof of Theorem \ref{thm:cfg_ib}, let $L(t)$ denote the information bottleneck objective defined in \eqref{eq:IBDiff}.
If $t^*$ is an interior minimizer of $L(t)$, it must satisfy the first-order optimality condition
\begin{align}
    \frac{\rd L(t^*)}{\rd t^*} = 0.
\end{align}
Applying Lemma~\ref{thm:flow_fisher} to both mutual information terms appearing in $L(t)$, we obtain
\begin{align}
    \beta J(X_{t^*};C) = J(X_{t^*};X_0),
\end{align}
which completes the proof.
\end{proof}

\section{Estimation of Score on Rectified Flow} \label{appen:discuss}
\subsection{Asymptotic Consistency Near the Gaussian Endpoint}
The score--velocity conversion used in flow matching \cite{Lipman2022FlowMF} is exact for stochastic Gaussian interpolation paths with independent endpoints. In particular, for independent $(X_0, X_1)$, one has
\[
\nabla \log p_t(x)
=
\frac{\frac{\alpha_t}{\dot \alpha_t} v_t(x)-x}
{\sigma_t^2-\frac{\alpha_t}{\dot \alpha_t}\dot \sigma_t \sigma_t}.
\]
For the rectified flow parameterization \cite{Liu2022FlowSA}, this identity reduces to the surrogate score
\[
\tilde s_t(x)
=
\frac{(t-1)v_t(x)-x}{t}.
\]
For deterministic couplings such as $n$-rectified flow with $n\ge 2$, the above conversion is in general no longer exact, since the information-theoretic assumptions underlying Theorem~2 do not apply. Therefore, in this setting, $\tilde s_t$ should be understood as a practical approximation induced by the learned velocity field, rather than the true score $s_t=\nabla \log p_t$.

Nevertheless, this approximation becomes asymptotically accurate near the high-noise Gaussian endpoint. Indeed, as $t\to 1$,
\[
\tilde s_t(x) = \frac{(t-1)v_t(x)-x}{t} \longrightarrow -x,
\]
provided that $v_t(x)$ remains bounded near $t=1$. On the other hand, if the endpoint distribution satisfies $p_1=\mathcal N(0,I)$, then its true score is exactly
\[
s_1(x)=\nabla \log p_1(x)=-x.
\]
Hence,
\[
\tilde s_t(x)\to s_1(x)
\qquad \text{as } t\to 1.
\]
This implies that $\tilde s_t(x)$ is an asymptotically accurate proxy near the Gaussian endpoint.

\subsection{Consistency of Score Differences via Local PDE Dynamics}
We next compare the evolution of the two score functions. The surrogate score satisfies
\begin{align*}
\partial_t\tilde s_t
=
\frac{(1-t)v_t+x_t}{t^2}
+\frac{v_t+(t-1)\frac{\partial v_t}{\partial t}}{t}
=
\frac{v_t+x_t}{t^2}
+\frac{t-1}{t}\frac{\partial v_t}{\partial t},
\end{align*}
and when $t\to1$, it becomes
\begin{align}
\partial_t\tilde s_t \approx x_t+v_t.
\label{eq:approx}
\end{align}

Meanwhile, the true score satisfies
\begin{align*}
\partial_t s_t
=
-\nabla(v_t\cdot s_t)-\nabla(\nabla\cdot v_t)
=
-J(v_t)s_t-J(s_t)v_t-\nabla(\nabla\cdot v_t).
\end{align*}
Near the Gaussian endpoint, we use the local approximation \(s_t(x_t)\approx -x_t\), so that \(J(s_t)\approx -I\). If, in addition, \(v_t\) is locally affine in \(x\) near \(t=1\), then \(\nabla(\nabla\cdot v_t)\approx 0\), which yields
\begin{align*}
\partial_t s_t \approx J(v_t)x_t+v_t.
\end{align*}

Subtracting the surrogate dynamics~\eqref{eq:approx} from the true score evolution, the local discrepancy in the time derivative is given by
\begin{align}
    \partial_t(s_t-\tilde s_t)\approx (J(v_t)-I)x_t.
\end{align}
The equation above is fit for any conditional velocity. And as our target is to find out whether $$\|s_t(x_t|c)\|^2-\|s_t(x_t|\varnothing)\|^2\approx\|\tilde s_t(x_t|c)\|^2-\|\tilde s_t(x_t|\varnothing)\|^2.$$
To rigorously justify the approximation, we examine the evolution between the conditional and unconditional score differences.
Since both the true scores and their surrogates converge to the same value (specifically $-x_t$) as $t \to 1$, the critical factor determining the approximation quality is the evolution of their difference. We thus focus on the quantity:
\[
\Delta_t = (s_t(x_t|c) - s_t(x_t|\varnothing)) - (\tilde{s}_t(x_t|c) - \tilde{s}_t(x_t|\varnothing)).
\]

Leveraging the local dynamics derived earlier, the time derivative of this discrepancy is governed by the Jacobian term:
\begin{align*}
    &\partial_t\Delta_t=\partial_t \big( (s_t - s_{t,c}) - (\tilde{s}_t - \tilde{s}_{t,c}) \big)\approx \big(J(v_t) - J(v_{t,c})\big) x_t,\\
    &\Delta_1=(s_1(x_t|c) - s_1(x_t|\varnothing)) - (\tilde{s}_1(x_t|c) - \tilde{s}_1(x_t|\varnothing))=0.
\end{align*}

As $t\to1$, the input data is dominated by noise, effectively masking the conditional information and making velocity fields $v_t$ and $v_{t,c}$ become increasingly similar near the Gaussian endpoint \cite{Gao_2026_CVPR}. Thus we assert that:
\[
J(v_t) - J(v_{t,c}) \to 0 \quad \text{as} \quad t \to 1.
\]

Consequently, the discrepancy $\Delta(t)$ vanishes asymptotically. This implies that the rate of change of the score difference $\|s_{t,c}\|^2-\|s_t\|^2$ in \eqref{eq:Delta} is preserved under the surrogate approximation$\|\tilde s_{t,c}\|^2-\|\tilde s_t\|^2$. Therefore, the proposed surrogate model asymptotically captures the correct guidance signal, validating the intended approximation.

\section{Implementation Details and Code}
\label{app:implementation_details}

This appendix provides the detailed implementation blueprint of our proposed method. As our method is a \textit{training-free} and \textit{plug-and-play} algorithm, the core implementation consists of a lightweight wrapper function that modifies the sampling process. Below is the code for our key procedures for SDE samplers of flow matching method. ODE samplers and DDPM/DDIM method is similar to this.
\input{sec/list}

\section{Additional Analysis and Experiments} \label{appen:exp}
\subsection{Properties of Stopping Time }
We collected a large-scale dataset of stopping times $\{\tau_i\}_{i=1}^N$ for CFG activation on SiT-XL/2 ($\log\beta^*=5.9$). To characterize the stochastic nature of this stopping time, we fitted the empirical distribution of $\tau$ against several candidate probability distributions, including Gamma, Log-Normal, Exponential, and Weibull distributions.

Descriptive statistics of the collected dataset reveal a mean stopping time of $\bar{\tau} \approx 0.680$ with a standard deviation of $\sigma_{\tau} \approx 0.050$. This indicates that while the stopping time exhibits a clear central tendency, it retains a moderate degree of variability across different sampling trajectories.

The goodness of fit was evaluated using the log-likelihood metric. We find out that the Weibull distribution achieves the highest log-likelihood ($\mathcal{L} \approx 713.24$), significantly outperforming the Exponential ($\mathcal{L} \approx -268.89$) and Gamma ($\mathcal{L} \approx 675.64$) baselines. This suggests that the optimal stopping time $\tau$ in our diffusion sampling process follows a Weibull distribution with shape parameter $k \approx 16.78$ and scale parameter $\lambda \approx 0.70$.

\begin{figure*}[!hbt]
    \centering
    \includegraphics[width=0.8\linewidth]{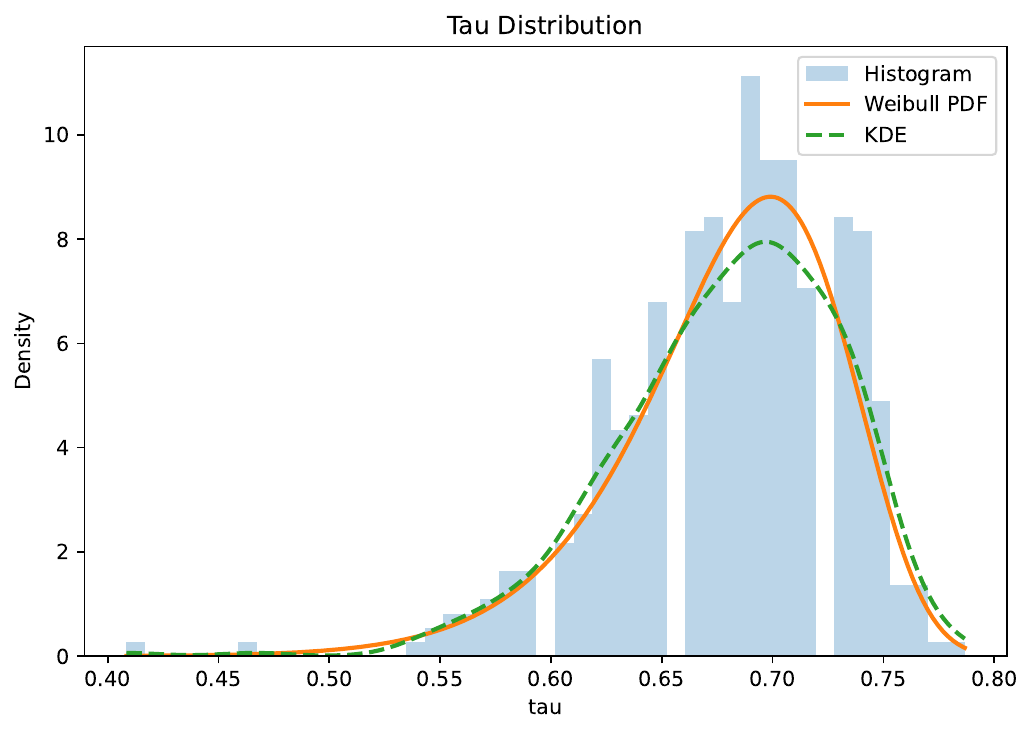}
    \caption{Probability density function (PDF) of the stopping time $\tau$. The histogram represents the empirical distribution of sampled $\tau$ values. The solid curve is the fitted Weibull PDF, and the dashed curve is the Kernel Density Estimation (KDE) curve.}
        \label{fig:tau_PDF}
\end{figure*}
The visual comparison presented in Figure~\ref{fig:tau_PDF} further corroborates the statistical findings. The histogram, which depicts the empirical frequency of $\tau$, aligns closely with both the Weibull PDF (solid line) and the KDE curve (dashed line). 

While the mean stopping time $\bar\tau$ closely approximates the heuristic interval boundary $t_h=0.7$ used in prior art (e.g., Interval Guidance), the distribution reveals a crucial nuance. The survival function of the Weibull fit suggests that although the peak intensity occurs around $0.7$, there exists a ``long tail" of activation events extending down to
$0.4$.

\input{tabel/appen_tab}

We additionally conduct an ablation by replacing the random stopping time with a fixed threshold set to the empirical mean $\bar{\tau} \approx 0.68$. The results are summarized in Table~\ref{tab:appen_tau_fid}. Compared to the baseline Interval Guidance ($t_h=0.7$), using the mean stopping time yields consistent improvements in IS, indicating that the statistical estimation of $\tau$ is indeed informative. However, this deterministic variant still underperforms our proposed stochastic stopping strategy, which achieves the best overall FID. 

This observation highlights that while the average stopping time captures the central tendency of activation, it fails to account for the intrinsic variability revealed by the Weibull distribution. In contrast, our method leverages this stochasticity to adaptively determine activation timing on a per-sample basis, thereby better aligning with the heterogeneous emergence of conditional structures during sampling. This result further supports the necessity and effectiveness of modeling $\tau$ as a random variable rather than a fixed constant.

\subsection{Ablation on Parameter \texorpdfstring{$\log \beta^*$}{log beta*}}
\input{tabel/ablation_beta}
We study the effect of the activation threshold $\log \beta^*$ in our dynamic CFG strategy on class-conditional ImageNet $256\times256$ generation. As shown in Table~\ref{tab:ablation4beta}, we vary $\log \beta^*$ while keeping other hyperparameters fixed.
We observe a clear non-monotonic trend. When $\log \beta^*$ is relatively small (e.g., 5.5–5.7), the performance is slightly worse than or comparable to the baseline, indicating that activating guidance too early introduces noisy or less informative conditional signals, which degrades sample quality. As $\log \beta^*$ increases (5.9–6.3), both FID and IS improve consistently, suggesting that delaying guidance allows the diffusion trajectory to first form a stable coarse structure before incorporating conditional information.
This trend is also visually summarized in Figure~\ref{fig:IS-FID-beta}, where the trajectory of points demonstrates the trade-off between quality and diversity, with the peak performance occurring around $\log\beta^* = 6.3$ before fidelity degrades at higher thresholds.

The best performance is achieved around $\log \beta^* \in [6.3, 6.7]$, where FID reaches 4.82 and IS improves significantly over the baseline. This aligns with our theoretical analysis that conditional information becomes more informative after a critical stage characterized by the Fisher-information balance.
However, when $\log \beta^*$ becomes too large (e.g., 6.9), we observe a slight degradation in FID despite continued improvement in IS. This suggests that overly delayed activation reduces the effective window during which guidance can refine samples, leading to suboptimal fidelity.

Overall, these results validate our hypothesis that guidance should be activated neither too early nor too late, and that an appropriate threshold plays a crucial role in achieving optimal performance.

\begin{figure*}[tbh]
    \centering
    \includegraphics[width=0.7\linewidth]{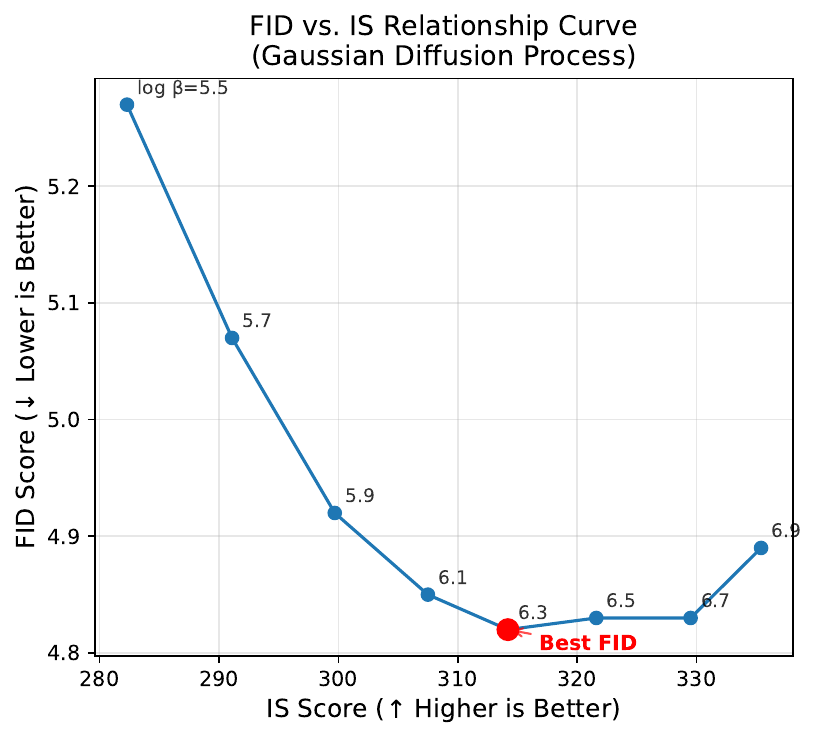}
    \caption{FID-IS trajectory under varying $\log\beta^*$ based on Table \ref{tab:ablation4beta}. Each point corresponds to a specific $\log\beta^*$ value along the trajectory. As $\log\beta^*$ increases, IS improves while FID exhibits a U-shaped trend, reflecting the quality-diversity trade-off in generation.}
        \label{fig:IS-FID-beta}
\end{figure*}
\subsection{Additional visual results}
Figure \ref{fig:vis} shows additional visual results using our method on Flux and SiT models.
\begin{figure}[!tbh]
    \centering
    \begin{subfigure}[H]{\textwidth}
        \centering
        \includegraphics[width=\linewidth]{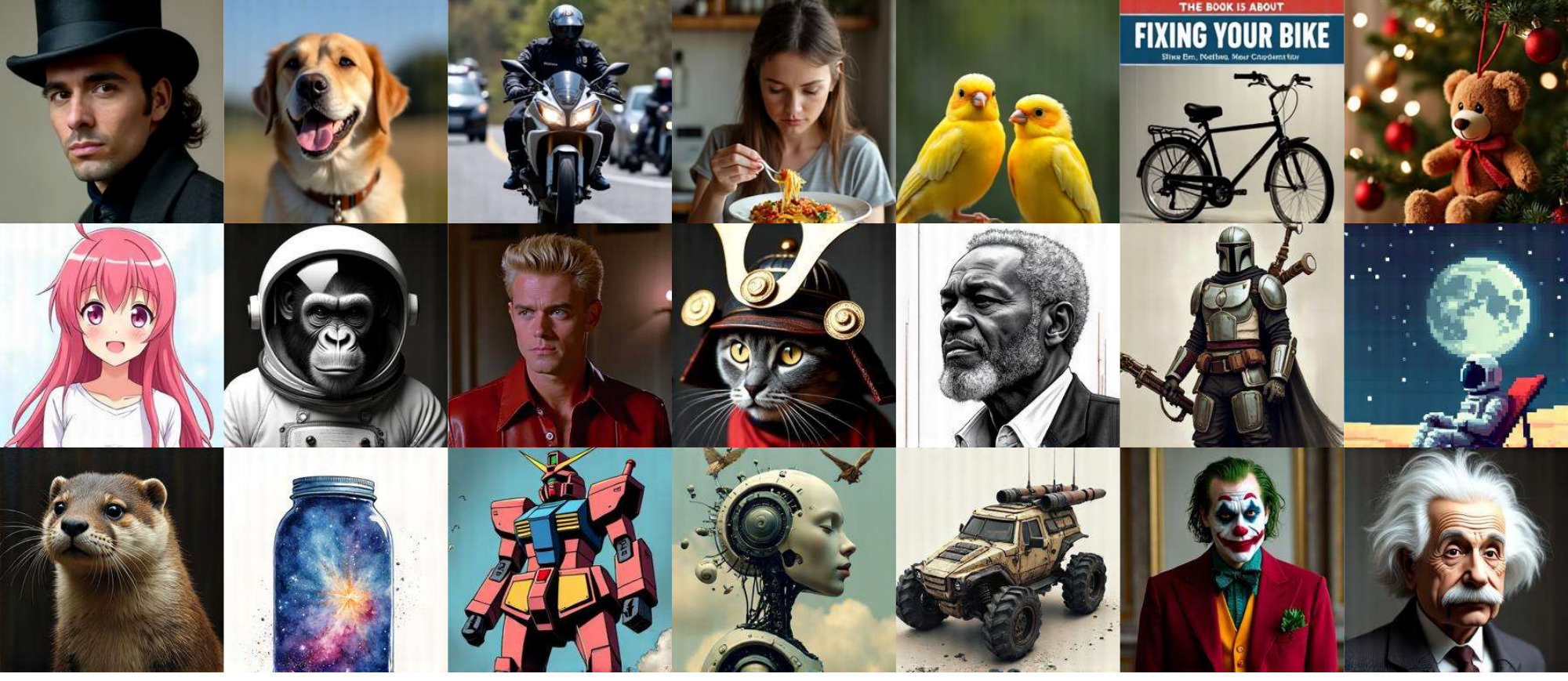}
        \caption{Images generated by the FLUX with our method.}
        \label{fig:vis_dit}
    \end{subfigure}
    
    
    \centering
    \begin{subfigure}[H]{\textwidth}
        \centering
        \includegraphics[width=\linewidth]{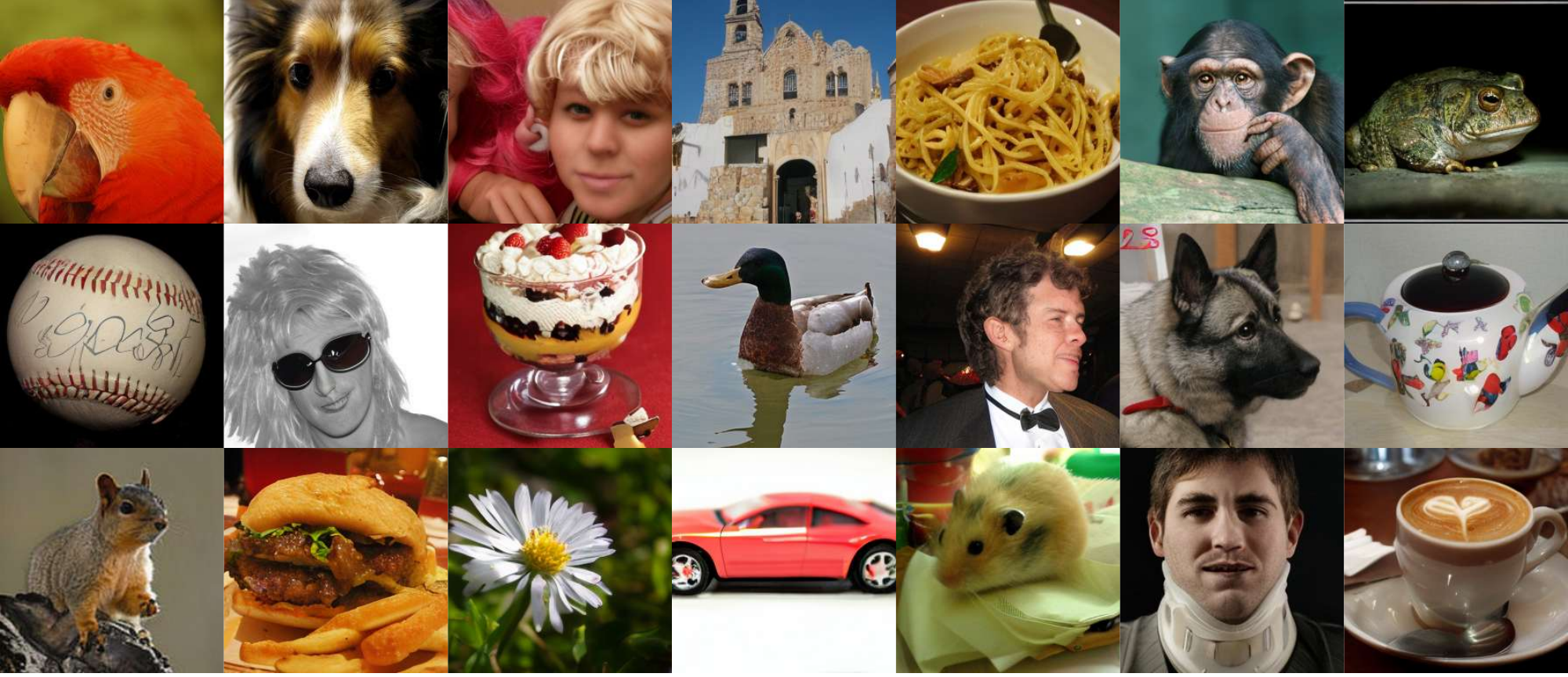}
        \caption{Images generated by the SiT-XL/2 (REPA) model with our method.}
        \label{fig:vis_repa}
    \end{subfigure}
    
    \caption{ 
         Additional visual  results on different models.
    }
    \label{fig:vis}
\end{figure}

%% file: sec/list.tex
\begin{center}
\begin{lstlisting}[language=Python, breaklines=true, basicstyle=\small\ttfamily]
def get_score(vt, xt, t, path_type="linear"):
    t = expand_t_like_x(t, xt)
    if path_type == "linear":
        alpha_t, d_alpha_t = 1 - t, torch.ones_like(xt, device=xt.device) * -1
        sigma_t, d_sigma_t = t, torch.ones_like(xt, device=xt.device)
    elif path_type == "cosine":
        alpha_t = torch.cos(t * np.pi / 2)
        sigma_t = torch.sin(t * np.pi / 2)
        d_alpha_t = -np.pi / 2 * torch.sin(t * np.pi / 2)
        d_sigma_t =  np.pi / 2 * torch.cos(t * np.pi / 2)
    else:
        raise NotImplementedError
    mean = xt
    reverse_alpha_ratio = alpha_t / d_alpha_t
    var = sigma_t - reverse_alpha_ratio * d_sigma_t
    score = (reverse_alpha_ratio * vt - mean)/ var
    return score
    
def euler_maruyama_sampler_plus(
        model,
        latents,
        y,
        num_steps=20,
        heun=False,  # not used, just for compatability
        cfg_scale=1.0,
        guidance_low=0.0,
        guidance_high=1.0,
        path_type="linear",
):
    thrd=-5.9
    if True:
        y_null = torch.tensor([1000] * y.size(0), device=y.device)
    _dtype = latents.dtype
    t_steps = torch.linspace(1., 0.04, num_steps, dtype=torch.float64)
    t_steps = torch.cat([t_steps, torch.tensor([0.], dtype=torch.float64)])
    x_next = latents.to(torch.float64)
    device = x_next.device
    flag = False
    with torch.no_grad():
        for i, (t_cur, t_next) in enumerate(zip(t_steps[:-2], t_steps[1:-1])):
            dt = t_next - t_cur
            x_cur = x_next
            model_input = torch.cat([x_cur] * 2, dim=0)
            y_cur = torch.cat([y, y_null], dim=0)
            kwargs = dict(y=y_cur)
            time_input = torch.ones(model_input.size(0)).to(device=device, dtype=torch.float64) * t_cur
            diffusion = compute_diffusion(t_cur)
            eps_i = torch.randn_like(x_cur).to(device)
            deps = eps_i * torch.sqrt(torch.abs(dt))
            v_cur = model(
                model_input.to(dtype=_dtype), time_input.to(dtype=_dtype), **kwargs
            )[0].to(torch.float64)
            s_cur = get_score_from_velocity(v_cur, model_input, time_input, path_type=path_type)
            d_cur = v_cur - 0.5 * diffusion * s_cur
            if t_cur <= guidance_high:
                flag=True
            if not flag:
                score = get_score(v_cur, model_input, time_input, path_type=path_type)
                s_cur_cond, s_cur_uncond = score.chunk(2)
                cond_sq = s_cur_cond.pow(2).mean(dim=[1, 2, 3])  # [B]
                uncond_sq = s_cur_uncond.pow(2).mean(dim=[1, 2, 3])
                E_cond = cond_sq.mean()
                E_uncond = uncond_sq.mean()
                beta = max(float(E_cond) - float(E_uncond), 0) + 1e-12
                log_beta = np.log(beta)
                if log_beta > thrd:
                    flag = True
            d_cur_cond, d_cur_uncond = d_cur.chunk(2)
            if flag:
                cfg_scale_cp = cfg_scale
            else:
                cfg_scale_cp = 1.0
            d_cur = d_cur_uncond + cfg_scale_cp * (d_cur_cond - d_cur_uncond)
            x_next = x_cur + d_cur * dt + torch.sqrt(diffusion) * deps
    t_cur, t_next = t_steps[-2], t_steps[-1]
    dt = t_next - t_cur
    x_cur = x_next
    if True:
        model_input = torch.cat([x_cur] * 2, dim=0)
        y_cur = torch.cat([y, y_null], dim=0)
    kwargs = dict(y=y_cur)
    time_input = torch.ones(model_input.size(0)).to(
        device=device, dtype=torch.float64
    ) * t_cur
    v_cur = model(
        model_input.to(dtype=_dtype), time_input.to(dtype=_dtype), **kwargs
    )[0].to(torch.float64)
    s_cur = get_score_from_velocity(v_cur, model_input, time_input, path_type=path_type)
    diffusion = compute_diffusion(t_cur)
    d_cur = v_cur - 0.5 * diffusion * s_cur
    if not flag:
        score = get_score(v_cur, model_input, time_input, path_type=path_type)
        s_cur_cond, s_cur_uncond = score.chunk(2)
        cond_sq = s_cur_cond.pow(2).mean(dim=[1, 2, 3])  # [B]
        uncond_sq = s_cur_uncond.pow(2).mean(dim=[1, 2, 3])
        E_cond = cond_sq.mean()
        E_uncond = uncond_sq.mean()
        beta = max(float(E_cond) - float(E_uncond), 0) + 1e-12
        log_beta = np.log(beta)
        if log_beta > thrd:
            flag = True
    d_cur_cond, d_cur_uncond = d_cur.chunk(2)
    if flag:
        cfg_scale_cp = cfg_scale
    else:
        cfg_scale_cp = 1.0
    d_cur = d_cur_uncond + cfg_scale_cp * (d_cur_cond - d_cur_uncond)
    mean_x = x_cur + dt * d_cur
    return mean_x
\end{lstlisting}
\end{center}

%% file: tabel/appen_tab.tex
\begin{table*}[!ht]
    \centering
    \footnotesize
    \setlength{\tabcolsep}{1pt}
    \caption{Comparison of fixed and stochastic stopping strategies on ImageNet 256. We evaluate a deterministic variant that uses the empirical mean stopping time ($t_h = \bar{\tau} \approx 0.68$) against Interval Guidance and our proposed stochastic stopping method. While the mean-based threshold improves over the baseline, it remains inferior to the stochastic strategy, highlighting the importance of modeling the variability of stopping time.}
    \resizebox{0.9\textwidth}{!}{%
    \begin{tabular}{lcccccc}
        \toprule
        \multicolumn{7}{l}{\bf\normalsize ImageNet  $B=48$ } \\
        \toprule
        \multirow{1}{*}{\bf{Model-256, 50k samples, 250 inference timesteps}} 
        & FID$\downarrow$ & IS$\uparrow$ & sFID$\downarrow$  \\
        \arrayrulecolor{black}\midrule
        SiT-XL/2 (REPA)($\omega=1.35$, SDE)     & 1.80 & 284.0 & \textbf{4.5} \\
        SiT-XL/2 (REPA, Interval) ($\omega=1.8, t_l=0,t_h=0.7$)     & \underline{1.42} & 305.7 &\underline{4.7} \\
        SiT-XL/2 (REPA) ($\omega=1.8, t_l=0,t_h=0.68$)     & {1.43} & {299.3} & \underline{4.7}   \\
        SiT-XL/2 (REPA) ($\omega=1.9, t_l=0,t_h=0.68$)     & \underline{1.42} & \textbf{310.4} & \underline{4.7}   \\
        \textbf{SiT-XL/2 (REPA) + Ours} ($\omega=1.9, \log \beta^*=5.9$)     & \textbf{1.41} & \underline{306.6} & \underline{4.7}   \\
        \bottomrule
    \end{tabular}
    }
    \label{tab:appen_tau_fid}
\end{table*}

%% file: tabel/ablation_beta.tex
\begin{table*}[!hbt]
\caption{\textbf{Ablation comparison on $\log \beta^*$.} Comparison of different evaluation metrics on Class-Conditional ImageNet datasets with different $\log \beta^*$.} 
    \begin{center}
    \scalebox{1.1}
    { 
    \begin{tabular}{lcccccc}
    \toprule
    \multicolumn{7}{l}{{\bf\normalsize ImageNet( 256$\times$256 ), 10k samples, 50 inference timesteps} } \\
    \toprule
    \multirow{1}{*}{\bf{SiT-XL/2 (REPA), SDE}} 
    & FID$\downarrow$ & IS $\uparrow$ \\
    \arrayrulecolor{black}\midrule 
    Baseline($\omega=1.8, t_h=0.7, t_l=0, B=32$)     & 5.09 & 287.
    4\\
    \arrayrulecolor{gray}\cmidrule(lr){1-7}
    Ours($\omega=1.9, \log\beta^*=5.5$)     & 5.27 & 282.3   \\
    Ours($\omega=1.9, \log\beta^*=5.7$)     & 5.07 & 291.1   \\
    Ours($\omega=1.9, \log\beta^*=5.9$)     & 4.92 & 299.7   \\
    Ours($\omega=1.9, \log\beta^*=6.1$)     & 4.85 & 307.5   \\
    Ours($\omega=1.9, \log\beta^*=6.3$)     & 4.82 & 314.2   \\
    Ours($\omega=1.9, \log\beta^*=6.5$)     & 4.83 & 321.6   \\
    Ours($\omega=1.9, \log\beta^*=6.7$)     & 4.83 & 329.5   \\
    Ours($\omega=1.9, \log\beta^*=6.9$)     & 4.89 & 335.4   \\
    \arrayrulecolor{black}\bottomrule
    \end{tabular}
    }
    
    \end{center}
\label{tab:ablation4beta}
\end{table*}

%% file: checklist.tex
\section*{NeurIPS Paper Checklist}

\begin{enumerate}

\item {\bf Claims}
    \item[] Question: Do the main claims made in the abstract and introduction accurately reflect the paper's contributions and scope?
    \item[] Answer: \answerYes{} 
    \item[] Justification: The main claims made in the abstract and introduction accurately reflect the paper's contributions and scope.

\item {\bf Limitations}
    \item[] Question: Does the paper discuss the limitations of the work performed by the authors?
    \item[] Answer: \answerYes{} 
    \item[] Justification: See limitations in Section \ref{theory} and future work in Section \ref{conclusion}.

\item {\bf Theory assumptions and proofs}
    \item[] Question: For each theoretical result, does the paper provide the full set of assumptions and a complete (and correct) proof?
    \item[] Answer: \answerYes{} 
    \item[] Justification: See Appendix \ref{proof}.

    \item {\bf Experimental result reproducibility}
    \item[] Question: Does the paper fully disclose all the information needed to reproduce the main experimental results of the paper to the extent that it affects the main claims and/or conclusions of the paper (regardless of whether the code and data are provided or not)?
    \item[] Answer: \answerYes{} 
    \item[] Justification: See Section \ref{experiments results} and other settings are just same as the baseline.

\item {\bf Open access to data and code}
    \item[] Question: Does the paper provide open access to the data and code, with sufficient instructions to faithfully reproduce the main experimental results, as described in supplemental material?
    \item[] Answer: \answerYes{} 
    \item[] Justification: Our proposed method is a training-free, plug-and-play algorithm that modifies the inference step of standard diffusion models. The core logic consists of a lightweight statistical estimation and a dynamic thresholding mechanism, which is simple to implement. To ensure reproducibility, we have provided the Python code in Appendix \ref{app:implementation_details}. Since our method acts as a wrapper around existing samplers, no complex retraining code is required.

\item {\bf Experimental setting/details}
    \item[] Question: Does the paper specify all the training and test details (e.g., data splits, hyperparameters, how they were chosen, type of optimizer) necessary to understand the results?
    \item[] Answer: \answerYes{} 
    \item[] Justification: See Section \ref{experiments results} and other settings are just same as the baseline.

\item {\bf Experiment statistical significance}
    \item[] Question: Does the paper report error bars suitably and correctly defined or other appropriate information about the statistical significance of the experiments?
    \item[] Answer: \answerYes{} 
    \item[] Justification: We analyze the distribution of $\tau$ and validate its effectiveness in Appendix \ref{appen:exp}.

\item {\bf Experiments compute resources}
    \item[] Question: For each experiment, does the paper provide sufficient information on the computer resources (type of compute workers, memory, time of execution) needed to reproduce the experiments?
    \item[] Answer: \answerYes{} 
    \item[] Justification: See settings in Section \ref{experiments results}.
    
\item {\bf Code of ethics}
    \item[] Question: Does the research conducted in the paper conform, in every respect, with the NeurIPS Code of Ethics \url{https://neurips.cc/public/EthicsGuidelines}?
    \item[] Answer: \answerYes{} 
    \item[] Justification: The research conducted in the paper conforms, in every respect, with the NeurIPS Code of Ethics.

\item {\bf Broader impacts}
    \item[] Question: Does the paper discuss both potential positive societal impacts and negative societal impacts of the work performed?
    \item[] Answer: \answerNA{}  
    \item[] Justification: There is no societal impact of the work performed.
    
\item {\bf Safeguards}
    \item[] Question: Does the paper describe safeguards that have been put in place for responsible release of data or models that have a high risk for misuse (e.g., pre-trained language models, image generators, or scraped datasets)?
    \item[] Answer: \answerNA{} 
    \item[] Justification: The paper poses no such risks.

\item {\bf Licenses for existing assets}
    \item[] Question: Are the creators or original owners of assets (e.g., code, data, models), used in the paper, properly credited and are the license and terms of use explicitly mentioned and properly respected?
    \item[] Answer: \answerYes{} 
    \item[] Justification: The creators or original owners of assets (e.g., code, data, models), used in the paper, properly credited and are the license and terms of use explicitly mentioned and properly respected.

\item {\bf New assets}
    \item[] Question: Are new assets introduced in the paper well documented and is the documentation provided alongside the assets?
    \item[] Answer: \answerNA{} 
    \item[] Justification: The paper does not release new assets.

\item {\bf Crowdsourcing and research with human subjects}
    \item[] Question: For crowdsourcing experiments and research with human subjects, does the paper include the full text of instructions given to participants and screenshots, if applicable, as well as details about compensation (if any)? 
    \item[] Answer: \answerNA{} 
    \item[] Justification: The paper does not involve crowdsourcing nor research with human subjects.

\item {\bf Institutional review board (IRB) approvals or equivalent for research with human subjects}
    \item[] Question: Does the paper describe potential risks incurred by study participants, whether such risks were disclosed to the subjects, and whether Institutional Review Board (IRB) approvals (or an equivalent approval/review based on the requirements of your country or institution) were obtained?
    \item[] Answer: \answerNA{} 
    \item[] Justification: The paper does not involve crowdsourcing nor research with human subjects.

\item {\bf Declaration of LLM usage}
    \item[] Question: Does the paper describe the usage of LLMs if it is an important, original, or non-standard component of the core methods in this research? Note that if the LLM is used only for writing, editing, or formatting purposes and does \emph{not} impact the core methodology, scientific rigor, or originality of the research, declaration is not required.
    \item[] Answer: \answerNA{} 
    \item[] Justification: The core method development in this research does not involve LLMs as any important, original, or non-standard components.

\end{enumerate}